\documentclass{article}

\usepackage{algorithm}
\usepackage{algorithmicx}
\usepackage[noend]{algpseudocode}
\usepackage{amsmath}
\usepackage{amsthm}
\usepackage{amsfonts}
\usepackage{amssymb}
\usepackage{color}
\usepackage{comment}
\usepackage{enumitem}
\usepackage{graphicx}
\usepackage[bookmarks=false]{hyperref}
\usepackage[utf8]{inputenc}
\usepackage{listings}
\usepackage{makeidx}  
\usepackage{subcaption}
\usepackage{url}
\usepackage{wrapfig}
\usepackage{xcolor}
\usepackage[pdf]{graphviz}
\usepackage{multirow}
\usepackage{refcount}
\usepackage{pbox}
\usepackage{xspace}
\usepackage{bbm}
\usepackage{booktabs}
\usepackage{stmaryrd}
\usepackage{todonotes}
\usepackage[graphicx]{realboxes}
\usepackage{tikz}
\tikzset{elliptic state/.style={draw,ellipse}}
\usetikzlibrary{shapes,shapes.geometric,arrows,fit,calc,positioning,automata,}
\usetikzlibrary{automata,arrows,decorations.pathmorphing,decorations.markings,positioning,shapes,shapes.geometric,arrows.meta,bending,fit,calc}
\usepackage{tikz-qtree}
\usepackage{float}
\usepackage{floatflt}
\graphicspath{{Figures/}}

\usepackage[accepted]{icml2023}
\newtheorem{theorem}{Theorem}[section]

\newtheorem{lemma}[theorem]{Lemma}

\newcommand{\code}[1]{{\texttt {#1}}}

\newcommand{\R}{\mathbb{R}}
\newcommand{\E}{\mathbb{E}}

\newcommand{\Z}{\mathbb{Z}}

\newcommand{\A}{\mathcal{A}}
\newcommand{\X}{\mathcal{X}}
\newcommand{\T}{\mathcal{T}}
\newcommand{\D}{\mathcal{D}}
\newcommand{\B}{\mathcal{B}}
\newcommand{\G}{\mathcal{G}}
\newcommand{\W}{\mathcal{W}}
\newcommand{\F}{\mathcal{F}}
\newcommand{\V}{\mathcal{V}}
\newcommand{\M}{\mathcal{M}}

\renewcommand{\L}{\mathcal{L}}
\renewcommand{\H}{\mathcal{H}}

\newcommand{\semantics}[1]{{\llbracket #1 \rrbracket}}
\newcommand{\norm}[1]{{\lVert #1 \rVert}}
\newcommand{\nmod}[1]{{\lvert #1 \rvert}}

\newcommand{\async}{\code{async}}
\newcommand{\targ}{\code{targ}}
\newcommand{\ext}[1]{\semantics{#1}}

\algrenewcommand\algorithmicindent{0.75em}

\icmltitlerunning{Robust Subtask Learning for Compositional Generalization}

\begin{document}

\twocolumn[
\icmltitle{Robust Subtask Learning for Compositional Generalization}




\begin{icmlauthorlist}
\icmlauthor{Kishor Jothimurugan}{yyy}
\icmlauthor{Steve Hsu}{yyy}
\icmlauthor{Osbert Bastani}{yyy}
\icmlauthor{Rajeev Alur}{yyy}
\end{icmlauthorlist}

\icmlaffiliation{yyy}{University of Pennsylvania}

\icmlcorrespondingauthor{Kishor Jothimurugan}{kishor@seas.upenn.edu}

\icmlkeywords{Machine Learning, ICML}

\vskip 0.3in
]



\printAffiliationsAndNotice{}  % leave blank if no need to mention equal contribution

\begin{abstract}
Compositional reinforcement learning is a promising approach for training policies to perform complex long-horizon tasks. Typically, a high-level task is decomposed into a sequence of subtasks and a separate policy is trained to perform each subtask. In this paper, we focus on the problem of training subtask policies in a way that they can be used to perform any task; here, a task is given by a sequence of subtasks. We aim to maximize the worst-case performance over all tasks as opposed to the average-case performance. We formulate the problem as a two agent zero-sum game in which the adversary picks the sequence of subtasks. We propose two RL algorithms to solve this game: one is an adaptation of existing multi-agent RL algorithms to our setting and the other is an asynchronous version which enables parallel training of subtask policies. We evaluate our approach on two multi-task environments with continuous states and actions and demonstrate that our algorithms outperform state-of-the-art baselines.
\end{abstract}

\section{Introduction}\label{sec:intro}
Reinforcement learning (RL) has proven to be a promising strategy for solving complex control tasks such as walking~\citep{fujimoto2018addressing}, autonomous driving~\citep{ivanov2021compositional}, and dexterous manipulation~\citep{akkaya2019solving}. However, a key challenge facing the deployment of reinforcement learning in real-world tasks is its high sample complexity---solving any new task requires training a new policy designed to solve that task. One promising solution is \emph{compositional reinforcement learning}, where individual \emph{options} (or \emph{skills}) are first trained to solve simple tasks; then, these options can be composed together to solve more complex tasks~\citep{lee2018composing, lee2021adversarial, ivanov2021compositional}. For example, if a driving robot learns how to make left and right turns and to drive in a straight line, it can then drive along any route composed of these primitives.

A key challenge facing compositional reinforcement learning is the generalizability of the learned options. In particular, options trained under one distribution of tasks may no longer work well if used in a new task, since the distribution of initial states from which the options are used may shift. An alternate approach is to train the options separately to perform specific subtasks, but options trained this way might cause the system to reach states from which future subtasks are hard to perform. One can overcome this issue by handcrafting rewards to encourage avoiding such states~\citep{ivanov2021compositional}, in which case they generalize well, but this approach relies heavily on human time and expertise.

We propose a novel framework that addresses these challenges by formulating the option learning problem as an adversarial reinforcement learning problem. At a high level, the adversary chooses the task that minimizes the reward achieved by composing the available options. Thus, the goal is to learn a set of \emph{robust options} that perform well across \emph{all} potential tasks. Then, we provide two algorithms for solving this problem. The first adapts existing ideas for using reinforcement learning to solve Markov games to our setting. Then, the second shows how to leverage the compositional structure of our problem to learn options in parallel at each step of a value iteration procedure; in some cases, by enabling such parallelism, we can reduce the computational cost of learning.

\begin{figure*}[t]
    \centering
    \begin{subfigure}{0.2\textwidth}
    \includegraphics[width=\textwidth]{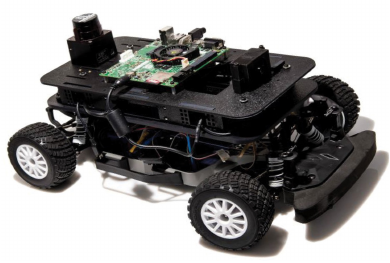}
    \caption{F1/10th Car}
    \end{subfigure}
    \qquad\qquad
    \begin{subfigure}{0.5\textwidth}
    \includegraphics[width=\textwidth]{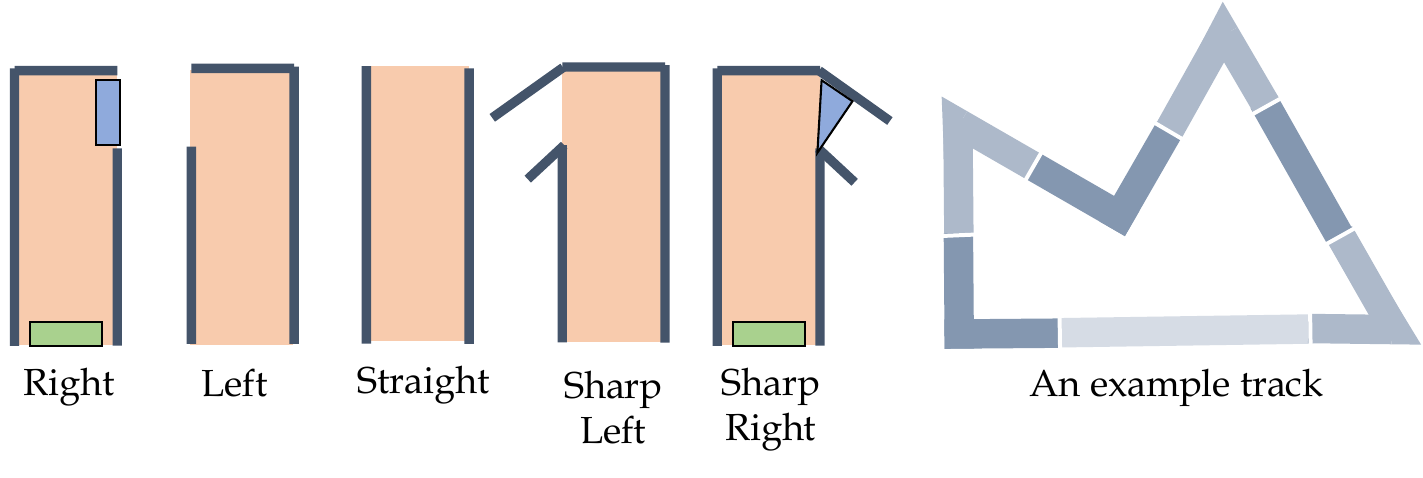}
    \caption{Segment Types}
    \end{subfigure}
    \caption{F1/10th Environment. The entry and exit regions for the right and sharp right segments are shown in green and blue respectively.}
    \label{fig:f110}
\end{figure*}

We validate our approach on two benchmarks: (i) a rooms environment where a point mass robot must navigate any given sequence of rooms, and (ii) a simulated version of the F1/10th car, where a small racing car must navigate any given racetrack that is composed of several different predefined track segments. In both environments, our empirical results demonstrate that robust options are critical for performing well on a wide variety of tasks.

In summary, our contributions are: (i) a game theoretic formulation of the compositional reinforcement learning problem, {(ii) an algorithm for solving this problem in the finite-state setting with guaranteed convergence in the limit,} (iii) two learning algorithms for continuous-state systems, and (iv) an empirical evaluation demonstrating the effectiveness of our approach.

\paragraph{Motivating example.}
Let us consider a small scale autonomous racing car shown in Figure~\ref{fig:f110} (a). We would like to train a controller that can be used to navigate the car through \emph{all} tracks constructed using five kinds of segments; the possible segments are shown in Figure~\ref{fig:f110} (b) along with an example track. A state of the car is a vector $s=(x,y,v,\theta)$ where $(x,y)$ is its position on the track relative to the current segment, $v$ is its current speed and $\theta$ is the heading angle. At any state $s$, the controller observes the measurements $o(s)\in\R^{1081}$ of a LiDAR sensor and outputs an action $(a,\omega)\in\R^2$ where $a$ is the throttle input and $\omega$ is the steering angle. In this environment, completing each segment is considered a subtask and a task corresponds to completing a sequence of segments (which may describe a track)---e.g., $\code{straight}\to\code{right}\to\code{left}\to\code{sharp-right}$. Upon completion of a subtask, the car enters the next segment and a change-of-coordinates\footnote{The change-of-coordinates is assumed for simpler modelling and does not affect the LiDAR measurements or the controller.} is applied to the car's state which is now relative to the new segment. The goal here is to learn one option for each subtask such that the agent can perform any task using these options.

If one trains the options independently with the only goal of reaching the end of each segment (e.g., using distance-based rewards), it might (and does) happen that the car reaches the end of a segment in a state that was not part of the initial states used to train the policy corresponding to the next segment. Therefore, one should make sure that the initial state distribution used during training includes such states as well---either manually or using dataset aggregation~\citep{ross2011reduction}. Furthermore, it is possible that the car reaches a state in the exit region of a segment from which it is challenging to complete the next subtask---e.g., a state in which the car is close to and facing towards a wall. Our algorithm identifies during training that, in order to perform future subtasks, it is better to reach the end of a segment in a configuration where the car is facing straight relative to the next segment. {Finally, we want the trained options to be such that they compose well for \emph{every} sequence of segments---i.e., track geometry. Therefore, we are interested in maximizing the worst case performance with respect to the choice of high-level task.

Our framework is broadly applicable in many real-world scenarios. For instance, the F1/10th example can be seen as a miniature version of an autonomous driving scenario where the agent needs to learn to perform maneuvers such as turning left/right, changing lanes etc. Here, the policies for performing these maneuvers need to ensure that the car is in a safe and controllable state for future maneuvers. Another interesting scenario is when a household robot has to perform multiple tasks such as cleaning, cooking etc., but needs to ensure that the policies for performing these tasks leave the house in a favorable state for future tasks---e.g., learning to cook without making the place too dirty (as it might be hard to perform the cleaning task later).}

\section{Problem Formulation}\label{sec:problem}
A \emph{multi-task Markov decision process (MDP)} is a tuple $\M = (S,A,P,\Sigma, R, F, T, \gamma,\eta,\sigma_0)$, where $S$ are the states, $A$ are the actions, $P(s'\mid s,a)\in[0,1]$ is the probability of transitioning from $s$ to $s'$ on action $a$, $\eta$ is the initial state distribution, and $\gamma\in(0,1)$ is the discount factor. Furthermore, $\Sigma$ is a set of subtasks and for each subtask $\sigma\in\Sigma$, $R_{\sigma}:S\times A\to\R$ is a reward function\footnote{We can also have a reward function $R_{\sigma}:S\times A\times S\to\R$ that depends on the next state but we omit it for clarity of presentation.}, $F_\sigma\subseteq S$ is a set of final states where the subtask is considered completed and $T_\sigma:F_{\sigma}\times S\to[0,1]$ is the jump probability function; upon reaching a state $s$ in $F_\sigma$ the system jumps to a new state $s'$ with probability $T_{\sigma}(s'\mid s)$. For the sake of clarity, we assume\footnote{This assumption can be removed by adding a fictitious copy of $F_{\sigma}$ to $S$ for each $\sigma\in\Sigma$.} that $T_\sigma(s'\mid s) = 0$ for any $s'$ with $s'\in F_{\sigma'}$ for some $\sigma'$. Finally, $\sigma_0\in\Sigma$ is the initial subtask which has to be completed first\footnote{When there is no fixed initial subtask, we can add a fictitious initial subtask.}. A multi-task MDP can be viewed as a discrete time variant of a hybrid automaton model~\citep{ivanov2021compositional}.

In the case of our motivating example, the set of subtasks is given by 
$
\Sigma=\{\code{left}, \code{right}, \code{straight},$ $\code{sharp-left}, \code{sharp-right}\}
$
with $F_{\sigma}$ denoting the exit region of the segment corresponding to subtask $\sigma$. We use the jump transitions $T$ to model the change-of-coordinates performed upon reaching an exit region. The reward function $R_\sigma$ for a subtask $\sigma$ is given by $R_{\sigma}(s,a,s') = -\norm{s'-c_{\sigma}}_2^2 + B\cdot\mathbbm{1}(s'\in F_{\sigma})$ where $c_{\sigma}$ is the center of the exit region and the subtask completion bonus $B$ is a positive constant. 

A task $\tau$ is defined to be an infinite sequence\footnote{A finite sequence can be appended with an infinite sequence of fictitious subtasks with zero reward.} of subtasks $\tau=\sigma_0\sigma_1\ldots$, and $\T$ denotes the set of all tasks. For any task $\tau\in\T$, $\tau[i]$ denotes the $i^{\text{th}}$ subtask $\sigma_i$ in $\tau$. In our setting, the task is chosen by the environment nondeterministically. Given a task $\tau$, a configuration of the environment is a pair $(s,i)\in S\times\Z_{
\geq 0}$ with $s\notin F_{\tau[i]}$ denoting that the system is in state $s$ and the current subtask is $\tau[i]$. The initial distribution over configurations $\Delta:S\times\Z_{\geq 0}\to[0,1]$ is given by $\Delta(s, i) = \eta_{\tau[0]}(s)$ if $i=0$ and $0$ otherwise. The probability of transitioning from $(s, i)$ to $(s', j)$ on an action $a$ is given by $\Pr((s',j)\mid (s,i), a) =$
\begin{align*}
\begin{cases}
\sum_{s''\in F_{\tau[i]}}P(s''\mid s, a)T_{\tau[i]}(s'\mid s'') &\text{if}\ j=i+1\\
P(s'\mid s, a) \qquad\qquad\qquad\qquad\ \ \text{if}\ j=i\ &\&\ s'\notin F_{\tau[i]}\\
0 &\text{otherwise.}
\end{cases}
\end{align*}

Intuitively, the system transitions to the next subtask if the current subtask is completed and stays in the current subtask otherwise. A (deterministic) policy is a function $\pi:S\to A$, where $a=\pi(s)$ is the action to take in state $s$. Our goal is to learn one policy $\pi_{\sigma}$ for each subtask $\sigma$ such that the discounted reward over the worst-case task $\tau$ is maximized. Formally, given a set of policies $\Pi=\{\pi_\sigma\ |\ \sigma\in\Sigma\}$ and a task $\tau$, we can define a Markov chain over configurations with initial distribution $\Delta$ and transition probabilities given by $P_{\Pi}((s',j)\mid (s,i)) = \Pr((s,j')\mid(s,i), \pi_{\tau[i]}(s))$. We denote by $\D_{\tau}^{\Pi}$ the distribution over infinite sequences of configurations $\rho=(s_0,i_0)(s_1,i_1)\ldots$ generated by $\tau$ and $\Pi$. Then, we define the objective function as
$$J(\Pi) = \inf_{\tau\in\T}\E_{\rho\sim\D_{\tau}^{\Pi}}\Big[\sum_{t=0}^\infty\gamma^tR_{\tau[i_t]}(s_t, \pi_{\tau[i_t]}(s_t))\Big].$$
These definitions can be naturally extended to stochastic policies as well. In our motivating example, choosing a large enough completion bonus $B$ guarantees the discounted reward to be higher for runs in which more subtasks are completed. Our aim is to compute a set of policies
$\Pi^*\in\arg\max_{\Pi}J(\Pi).$
%Note that we are interested in the worst-case performance of $\Pi$ as opposed to the average performance over all tasks which is commonly studied in multi-task RL.
Each subtask policy $\pi_{\sigma}$ defines an option~\citep{sutton1999between} $o_{\sigma} = (\pi_{\sigma}, I_{\sigma}, \beta_{\sigma})$ where $I_{\sigma} = S\setminus F_{\sigma}$ and $\beta_{\sigma}(s) = \mathbbm{1}(s\in F_{\sigma})$. Here, the choice of which option to trigger is made by the environment rather than the agent.

\section{Reduction to Stagewise Markov Games}\label{sec:reduction}
The problem statement naturally leads to a game theoretic view in which the environment is the adversary. We can formally reduce the problem to a two-agent zero-sum Markov game $\G = (\bar{S}, A_1, A_2, \bar{P}, \bar{R}, \bar{\gamma},\bar{\eta})$ where $\bar{S} = S\times\Sigma$ is the set of states,
$A_1=A$ are the actions of agent $1$ {(the agent representing the options)} and $A_2=\Sigma$ are the actions of agent $2$ {(the adversary)}. The transition probability function $\bar{P}:\bar{S}\times A_1\times A_2\times \bar{S}\to[0,1]$ is given by $\bar{P}((s',\sigma')\mid(s,\sigma), a_1, a_2) =$
\begin{align*}
\begin{cases}
P(s'\mid s,a_1) &\text{if}\ s \notin F_{\sigma}\ \&\ \sigma=\sigma'\\
T_{\sigma}(s'\mid s) &\text{if}\ s \in F_{\sigma}\ \&\ \sigma'= a_2\\
0 &\text{otherwise.}
\end{cases}
\end{align*}
We observe that the states are partitioned into two sets $\bar{S}=S_1\cup S_2$ where $S_1 = \{(s,\sigma)\mid s\notin F_{\sigma}\}$ is the set of states where agent $1$ acts (causing a step in $\M$) and $S_2 = \{(s,\sigma)\mid s\in F_{\sigma}\}$ is the set of states where agent $2$ takes actions (causing a change of subtask); this makes $\G$ a stagewise game. The reward function $\bar{R}:\bar{S}\times A_1\to \R$ is given by $\bar{R}((s,\sigma), a) = R_{\sigma}(s,a)$ if $s\notin F_{\sigma}$ and $0$ otherwise. The discount factor depends on the state and is given by $\bar{\gamma}(s,\sigma) = \gamma$ if $s \notin F_{\sigma}$ and $1$ otherwise; this is because a change of subtask does not invoke a step in $\M$. The initial state distribution $\bar{\eta}$ is given by $\bar{\eta}(s,\sigma) = \eta(s)\mathbbm{1}(\sigma=\sigma_0)$. A run of the game is a sequence $\bar{\rho} = \bar{s}_0a_0^1a_0^2\bar{s}_1a_1^1a_1^2\ldots$ where $\bar{s}_t\in\bar{S}$ and $a_t^i\in A_i$.

A (deterministic) policy for agent $i$ is a function ${\pi}_i:\bar{S}\to A_i$. Given policies $\pi_1$ and $\pi_2$ for agents 1 and 2, respectively and a state $\bar{s}\in\bar{S}$ we denote by
$\D^{\G}_{\bar{s}}({\pi}_1,{\pi}_2)$ the distribution over runs generated by ${\pi}_1$ and ${\pi}_2$ starting at $\bar{s}$. Then, the value of a state $\bar{s}$ is defined by
$$V^{{\pi}_1, {\pi}_2}(\bar{s}) = \E_{\bar{\rho}\sim\D^{\G}_{\bar{s}}({\pi}_1,{\pi}_2)}\Big[\sum_{t=0}^{\infty}\big(\prod_{k=0}^{t-1}\bar{\gamma}(\bar{s}_k)\big)\bar{R}(\bar{s}_t,a_t^1)\Big].$$
We are interested in computing a policy $\pi_1^*$ maximizing $$J_\G(\pi_1) = \E_{\bar{s}\sim\bar{\eta}}[\min_{\pi_2}V^{\pi_1,\pi_2}(\bar{s})].$$

Given a policy ${\pi}_1$ for agent 1, we can construct a policy $\pi_{\sigma}$ for any subtask $\sigma$ given by $\pi_{\sigma}(s) = {\pi}_1(s,\sigma)$; we denote by $\Pi({\pi}_1)$ the set of subtask policies constructed this way. The following theorem connects the objective of the game with our multi-task learning objective; all proofs are in Appendix~\ref{sec:proofs}.
\begin{theorem}\label{thm:reduction}
For any policy $\pi_1$ of agent 1 in $\G$, we have $J(\Pi(\pi_1))\geq J_\G(\pi_1).$
\end{theorem}
Therefore, $J_{\G}(\pi_1)$ is a lower bound on the objective $J(\Pi(\pi_1))$ which we seek to maximize. Now, let us define the optimal value of a state $\bar{s}$ by
$V^*(\bar{s}) = \max_{{\pi}_1}\min_{{\pi}_2}V^{{\pi}_1, {\pi}_2}(\bar{s})$. The following theorem shows that it is possible to construct a policy $\pi_1^*$ that maximizes $J_{\G}(\pi_1)$ from the optimal value function $V^*$.

\begin{theorem}\label{thm:opt_policy}
For any policy $\pi_1^*$ such that for all $(s,\sigma)\in S_1$, $\pi_1^*(s,\sigma)\in\operatorname{\arg\max}_{a\in A}\Big\{\bar{R}((s,\sigma), a) + \gamma\cdot\sum_{s'\in S}P(s'\mid s,a){V}^*(s',\sigma)\Big\},$
we have that $\pi_1^*\in\arg\max_{\pi_1}J_\G(\pi_1)$.
\end{theorem}

\subsection{Value Iteration}\label{sec:vi}
In this section, we briefly look at two value iteration algorithms to compute $V^*$ which we later adapt in
Section~\ref{sec:algorithms} to obtain learning algorithms. Let $\V=\{V:S_1\to\R\}$ denote the set of all value functions over $S_1$. Given a value function $V\in\V$ we define its extension to all of $\bar{S}$ using $\ext{V}(s,\sigma) =$
\begin{equation}\label{eq:ext}
\begin{cases}
\min_{\sigma'\in\Sigma}\sum_{s'\in S}T_\sigma(s'\mid s)V(s',\sigma') &\text{if}\ s\in F_{\sigma}\\
V(s,\sigma) &\text{otherwise.}
\end{cases}
\end{equation}
For a state $s\in F_{\sigma}$, $\ext{V}(s,\sigma)$ denotes the worst-case value (according to $V$) with respect to the possible choices of next subtask $\sigma'$. Now, we consider the Bellman operator $\F:\V\to\V$ defined by $\F(V)(s,\sigma) = $
\begin{equation}\label{eq:vi}
\max_{a\in A}\Big\{\bar{R}((s,\sigma), a) + \gamma\cdot\sum_{s'\in S}P(s'\mid s,a)\ext{V}(s',\sigma)\Big\}   
\end{equation}

for all $(s,\sigma)\in S_1$. Let us denote by $V^*\downarrow_{S_1}$ the restriction of $V^*$ to $S_1$.
The following lemma follows straightforwardly, giving us our first value iteration procedure.
\begin{theorem}\label{thm:fp}
$\F$ is a contraction mapping with respect to the $\ell_{\infty}$-norm on $\V$ and $V^*\downarrow_{S_1}$ is the unique fixed point of $\F$ with $\lim_{n\to\infty}\F^n(V)=V^*\downarrow_{S_1}$ for all $V\in\V$.
\end{theorem}

\begin{algorithm}[t]
\begin{algorithmic}[1]
\Function{AsyncValueIteration}{$\M,V$}
\While{stopping criterion is met}
\For{$\sigma\in\Sigma$} \quad{\color{blue}// in parallel}
\State Compute $\W_{\sigma}(V)$
\EndFor
\State $V\gets\F_{\async}(V)$ \quad{\color{blue}// using Equation~\ref{eq:async_vi}}
\EndWhile
\EndFunction
\end{algorithmic}
\caption{Asynchronous value iteration algorithm for computing optimal subtask policies.}
\label{alg:async_vi}
\end{algorithm}

\paragraph{Asynchronous VI.} Next we consider an \emph{asynchronous} value iteration procedure which allows us to parallelize computing subtask policies for different subtasks. Given a subtask $\sigma$ and a value function $V\in\V$, we define a \emph{subtask MDP} $\M_{\sigma}^V$ which behaves like $\M$ until reaching a final state $s\in F_{\sigma}$ after which it transitions to a dead state $\bot$ achieving a reward of $\ext{V}(s,\sigma)$. Formally, $\M^V_{\sigma} = (S_{\sigma},A,P_{\sigma},R_{\sigma}^V,\gamma)$ where $S_{\sigma}=S\sqcup\{\bot\}$ with $\bot$ being a special dead state. The transition probabilities are $$P_\sigma(s'\mid s,a) = \begin{cases}
    P(s'\mid s,a) &\text{if}\ \bot\neq s\notin F_\sigma\\
    \mathbbm{1}(s'=\bot) &\text{otherwise.}
\end{cases}$$
The reward function is given by $$R_{\sigma}^V(s,a) = \begin{cases}
    R_{\sigma}(s,a) &\text{if}\ \bot\neq s\notin F_{\sigma}\\
    \ext{V}(s,\sigma) &\text{if}\ \bot\neq s\in F_{\sigma}\\
    0 &\text{otherwise.}
\end{cases}$$
We denote by $\W_{\sigma}(V)$ the optimal value function of the MDP $\M_\sigma^V$. We then define the asynchronous value iteration operator $\F_{\async}:\V\to\V$ using
\begin{equation}\label{eq:async_vi}
\F_{\async}(V)(s,\sigma) = \W_{\sigma}(V)(s).
\end{equation}

We can show that repeated application of $\F_{\async}$ leads to the optimal value function $V^*$ of the game $\G$.
\begin{theorem}\label{thm:async}
For any $V\in\V$, $\lim_{n\to\infty}\F^n_{\async}(V)\to V^*\downarrow_{S_1}$.
\end{theorem}

Since each $\W_{\sigma}(V)$ can be computed independently, we can parallelize the computation of $\F_{\async}$ giving us the algorithm in Algorithm~\ref{alg:async_vi}. We can also show that it is not necessary to compute $\W_{\sigma}(V)$ exactly. Let $\V_{\sigma} = \{\bar{V}:S_\sigma\to\R\}$ be the set of all value functions over $S_{\sigma}$. For a fixed $V\in\V$, let $\F_{\sigma,V}:\V_{\sigma}\to\V_{\sigma}$ denote the usual Bellman operator for the MDP $\M_{\sigma}^V$ given by $\F_{\sigma,V}(\bar{V})(s) = $
$$\max_{a\in A}\Big\{R^V_{\sigma}(s,a) + \gamma\cdot\sum_{s'\in S_{\sigma}}P_{\sigma}(s'\mid s,a)\bar{V}(s')\Big\}$$
for all $\bar{V}\in\V_{\sigma}$ and $s\in S_{\sigma}$. Now for any $V\in\V$ and $\sigma\in\Sigma$, we define a corresponding ${V}_\sigma\in\V_{\sigma}$ using ${V}_{\sigma}(s) = \ext{V}(s,\sigma)$ if $s\in S$ and ${V}_{\sigma}(\bot) = 0$. Then, for any integer $m>0$ and $V\in\V$, we can use $\F_{\sigma,V}^m({V}_\sigma)$ as an approximation to $\W_{\sigma}(V)$. Let us define $\F_m:\V\to\V$ using
$$\F_m(V)(s,\sigma) = \F_{\sigma,V}^m({V}_\sigma)(s).$$
Intuitively, $\F_m$ corresponds to performing $m$ steps of value iteration in each subtask MDP $\M_{\sigma}^V$ (which can be parallelized) starting at $V_{\sigma}$. The following theorem guarantees convergence when using $\F_m$ instead of $\F_{\async}$.

\begin{theorem}\label{thm:async_partial}
For any $V\in\V$ and $m>0$, $\lim_{n\to\infty}\F^n_{m}(V)\to V^*\downarrow_{S_1}$.
\end{theorem}

\section{Learning Algorithms}\label{sec:algorithms}
In this section, we present RL algorithms for solving the game $\G$. We first consider the finite MDP setting for which we can obtain a modified $Q$-learning algorithm with a convergence guarantee. We then present two algorithms based on Soft Actor Critic (\textsc{Sac}) for the continuous-state setting.

\subsection{Finite MDP}
Assuming finite states and actions, we can obtain a $Q$-learning variant for solving $\G$ which we call \emph{Robust Option $Q$-learning}. We assume that jump transitions $T$ are known to the learner; this is usually the case since jump transitions are used to model subtask transitions and (potential) change-of-coordinates within the controller. However, we believe that the algorithm can be easily extended to the scenario where $T$ is unknown.

We maintain a function $Q:S_1\times A\to\R$ with $Q(s,\sigma,a)$ denoting $Q((s,\sigma), a)$. The corresponding value function $V_Q$ is defined using $V_Q(s,\sigma) = \max_{a\in A}Q(s,\sigma, a)$ and is extended to all of $\bar{S}$ as $\ext{V_Q}$. Note that, given a $Q$-function, the extended value function $\ext{V_Q}$ can be computed exactly. Robust Option $Q$-learning is an iterative process---in each iteration $t$, it takes a step $((s,\sigma), a_1, a_2, (s',\sigma))$ in $\G$ with $(s,\sigma)\in S_1$ and updates the $Q$-function using
\begin{align*}\label{eq:qlearning}
Q_{t+1}(s,\sigma, a_1)\gets (1-&\alpha_t)Q_t(s,\sigma,a_1)\\ +&\  \alpha_t (\bar{R}((s,\sigma),a_1) + \gamma \ext{V_{Q_t}}(s',\sigma)).
\end{align*}
where $Q_t$ is the $Q$-function in iteration $t$ and $\ext{V_{Q_t}}$ is the corresponding extended value function.

% \begin{algorithm}[t]
% \begin{algorithmic}[1]
% \Function{QLearn}{$\bar{\pi}_1,\bar{\pi}_2, \eta_\G, \alpha$}
% \State Initialize $Q$-function $Q_1$
% \State Sample $(s,\sigma)\sim\eta_\G$
% \For{$t \in \{1,2,\ldots\}$}
% \State $a\gets\bar{\pi}_1^{Q_t}(s,\sigma)$
% \State $s'\sim P(s,a,\cdot)$
% \If{$s'\notin F_{\sigma}$}
% \State $v_t\gets\max_{a\in A}Q((s',\sigma), a)$
% \Else 
% \State $v_t\gets\min_{\sigma'\in\Sigma}\sum_{s''\in S}T_{\sigma}(s',s'')\big(\max_{a\in A}Q_t((s'',\sigma'), a)\big)$
% \EndIf
% \State $Q_{t+1}((s,\sigma), a)\xleftarrow{\alpha_t} \bar{R}((s,\sigma),a) + \gamma v_t$ \quad{\color{blue}// Equation~\ref{eq:qlearning}}
% \If{$s'\notin F_{\sigma}$}
% \State $s\gets s'$
% \Else
% \State $s\sim T_{\sigma}(s',\cdot)$ and $\sigma\gets \bar{\pi}_2^{Q_{t+1}}(s',\sigma)$
% \EndIf
% \EndFor
% \EndFunction
% \end{algorithmic}
% \caption{Option $Q$-learning.\\
% Inputs: Exploration strategies $\bar{\pi}_1,\bar{\pi}_2$, initial distribution $\eta_\G$ and learning rate schedule $\alpha$.}
% \label{alg:qlearning}
% \end{algorithm}

%we use $\bar{\pi}_1^Q$ and $\bar{\pi}_2^Q$ to denote exploration policies in $\G$ for the two players respectively; they can depend on the current $Q$ function, e.g., $\varepsilon$-greedy.
% The distribution over $\bar{S}$ is given by $\eta_\G$ which, for example, could denote sampling $s\sim\eta$ and $\sigma\sim\code{Uniform}(\Sigma)$ independently. The algorithm maintains the invariant that $s\notin F_{\sigma}$.
Under standard assumptions on the learning rates $\alpha_t$, similar to $Q$-learning, we can show that Robust Option $Q$-learning converges to the optimal $Q$-function almost surely. Here, the optimal $Q$-function is defined by $Q^*(s,\sigma,a) = \bar{R}((s,\sigma), a) + \gamma\sum_{s'\in S}P(s'\mid s,a)V^*(s',\sigma)$ for all $(s,\sigma)\in S_1$. Let $\alpha_t(s,\sigma,a)$ denote the learning rate used in iteration $t$ if $Q_t(s,\sigma,a)$ is updated and $0$ otherwise. Then, we have the following theorem.

\begin{theorem}\label{thm:q_learning}
If $\sum_t\alpha_t(s,\sigma,a)=\infty$ and $\sum_t\alpha_t^2(s,\sigma,a)<\infty$ for all $(s,\sigma)\in S_1$ and $a\in A$, then $\lim_{t\to\infty}Q_t = Q^*$ almost surely.
\end{theorem}

\subsection{Continuous States and Actions}
\begin{algorithm}[t]
\begin{algorithmic}[1]
\Function{Rosac}{$\alpha_\psi$, $\alpha_\theta$, $\beta$, $\delta$}
\State Initialize params $\{\psi_{\sigma}\}_{\sigma\in\Sigma}$, $\{\psi_{\sigma}^{\targ}\}_{\sigma\in\Sigma}$, $\{\theta_{\sigma}\}_{\sigma\in\Sigma}$
\State Initialize replay buffer $\B$
\For{each iteration}
\For{each episode}
\State $s_0\sim \eta$
\State $\sigma_0\gets\code{InitialSubtask}$
\For{each step $t$}
\State $a_t\sim\pi_{\theta_{\sigma_t}}(\cdot\mid s_t)$
\State $s_{t+1}\sim P(\cdot\mid s,a)$
\State $\B\gets \B\cup\{(s_t,a_t,s_{t+1})\}$
\If{$s_{t+1}\in F_{\sigma_t}$}
\State $s_{t+1}\sim T_{\sigma_t}(\cdot\mid s_{t+1})$
\State $\sigma_{t+1}\gets\code{Greedy}(\varepsilon, \arg\min_{\sigma}\tilde{V}(s_{t+1},\sigma), \Sigma)$\label{line:argmin}
\Else
\State $\sigma_{t+1}\gets\sigma_t$
\EndIf
\EndFor
\EndFor
\For{each gradient step}
\State Sample batch $B\sim\B$
\For{$\sigma\in\Sigma$}
\State $\psi_{\sigma}\gets\psi_{\sigma}-\alpha_\psi\nabla_{\psi_{\sigma}}\L_{Q}(\psi_{\sigma},B)$
\State $\theta_{\sigma}\gets\theta_{\sigma}-\alpha_\theta\nabla_{\theta_{\sigma}}\L_{\pi}(\theta_{\sigma},B)$
\State $\psi_{\sigma}^{\targ}\gets\delta\psi_{\sigma} + (1-\delta)\psi_{\sigma}^{\targ}$
\EndFor
\EndFor
\EndFor
\EndFunction
\end{algorithmic}
\caption{Robust Option Soft Actor Critic.\\
Inputs: Learning rates $\alpha_{\psi}$, $\alpha_\theta$, entropy weight $\beta$ and Polyak rate $\delta$.}
\label{alg:osac}
\end{algorithm}

In the case of continuous states and actions, we can adapt any $Q$-function based RL algorithm such as  Deep Deterministic Policy Gradients (\textsc{Ddpg})~\citep{lillicrap2016continuous} or Soft Actor Critic (\textsc{Sac})~\citep{haarnoja2018soft} to our setting. Here we present an \textsc{Sac} based algorithm that we call Robust Option \textsc{Sac} (\textsc{Rosac}) which is outlined in Algorithm~\ref{alg:osac}. This algorithm, like \textsc{Sac}, adds an entropy bonus to the reward function to improve exploration.

We maintain two $Q$-functions for each subtask $\sigma$, $Q_{\psi_{\sigma}}:S\to\R$ parameterized by $\psi_{\sigma}$ and a target function $Q_{\psi_{\sigma}^{\targ}}$ parameterized by $\psi_{\sigma}^{\targ}$. We also maintain a stochastic subtask policy $\pi_{\theta_{\sigma}}:S\to \D(A)$ for each subtask $\sigma$ where $\D(A)$ denotes the set of distributions over $A$. Given a step $(s,a,s')$ in $\M$ and a subtask $\sigma$ with $s\notin F_{\sigma}$, we define the target value by
$$y_{\sigma}(s, a, s') = {R}_{\sigma}(s,a) + \gamma \ext{V}(s',\sigma)$$
where the value $\ext{V}(s',\sigma)$ is estimated using $\tilde{V}(s',\sigma)= Q_{\psi_{\sigma}^{\targ}}(s',\tilde{a}) - \beta\log\pi_{\theta_{\sigma}}(\tilde{a}\mid s')$ with $\tilde{a}\sim\pi_{\theta_{\sigma}}(\cdot\mid s')$ if $s'\notin F_{\sigma}$. If $s'\in F_{\sigma}$, we estimate $\ext{V}(s',\sigma)$ using $\tilde{V}(s',\sigma)=\min_{\sigma'\in\Sigma}\tilde{V}(s'',\sigma')$ where $\tilde{V}(s'',\sigma') = Q_{\psi_{\sigma'}^{\targ}}(s'',\tilde{a}) - \beta\log\pi_{\theta_{\sigma'}}(\tilde{a}\mid s'')$ with $\tilde{a}\sim\pi_{\theta_{\sigma'}}(\cdot\mid s'')$ and $s''\sim T_{\sigma}(\cdot\mid s')$. Now, given a batch $B=\{(s,a,s')\}$ of steps in $\M$ we update ${\psi_{\sigma}}$ using one step of gradient descent corresponding to the loss $\L_Q(\psi_{\sigma}, B) = $
$$\frac{1}{|B|}\sum_{(s,a,s')\in B}(Q_{\psi_{\sigma}}(s,a) - y_{\sigma}(s,a,s'))^2$$
and the subtask policy $\pi_{\theta_{\sigma}}$ is updated using the loss $\L_{\pi}(\theta_{\sigma}, B) = $
$$-\frac{1}{|B|}\sum_{(s,a,s')\in B}\displaystyle\E_{\tilde{a}\sim\pi_{\theta_{\sigma}}(\cdot\mid s)}\big[Q_{\psi_{\sigma}}(s,\tilde{a}) - \beta\log\pi_{\theta_{\sigma}}(\tilde{a}\mid s)\big].$$
The gradient $\nabla_{\theta_{\sigma}}\L_{\pi}(\theta_{\sigma}, B)$ can be estimated using the reparametrization trick if, for example, $\pi_{\theta_{\sigma}}(\cdot\mid s)$ is a Gaussian distribution whose parameters are differentiable w.r.t. $\theta_{\sigma}$. We use Polyak averaging to update the target $Q$-networks $\{Q_{\psi_{
\sigma}^{\targ}} \mid \sigma\in\Sigma\}$.

Note that we do not train a separate policy for the adversary. During exploration, we use the $\varepsilon$-greedy strategy to select subtasks. We first estimate the ``worst" subtask for a state $s$ using $\tilde{\sigma}=\arg\min_{\sigma}\tilde{V}(s,\sigma)$ where $\tilde{V}(s,\sigma)$ is estimated as before. Then the function $\code{Greedy}(\varepsilon,\tilde{\sigma}, \Sigma)$ chooses $\tilde{\sigma}$ with probability $1-\varepsilon$ and picks a subtask uniformly at random from $\Sigma$ with probability $\varepsilon$. {Furthermore, we can easily impose constraints on the sequences of subtasks considered (e.g., only consider realistic tracks in F1/10th) by restricting the adversary to pick the next subtask $\sigma_{t+1}$ from a subset $\Sigma_{t+1}\subseteq\Sigma$ of possible substasks---i.e., by performing $\arg\min$ (in Line~\ref{line:argmin}) over $\Sigma_{t+1}$ instead of $\Sigma$.}

\paragraph{Asynchronous \textsc{Rosac}.} We can also obtain an asynchronous version of the above algorithm which lets us train subtask policies in parallel. Asynchronous Robust Option \textsc{Sac} (\textsc{Arosac}) is outlined in Algorithm~\ref{alg:aosac}. Here we use one replay buffer for each subtask. We maintain an initial state distribution $\tilde{\eta}$ over $S$ to be used for training every subtask policy $\{\pi_{\sigma}\}_{\sigma\in\Sigma}$. $\tilde{\eta}$ is represented using a finite set of states $D$ from which a state is sampled uniformly at random. The value function $\tilde{V}:S\times \Sigma\to\R$ is estimated as before. To be specific, in each iteration, an estimate of any value $\tilde{V}(s,\sigma)$ is obtained on the fly using the $Q$-functions and the subtask policies from the previous iteration.

\begin{figure}[t]
    \centering
    \includegraphics[width=0.16\textwidth]{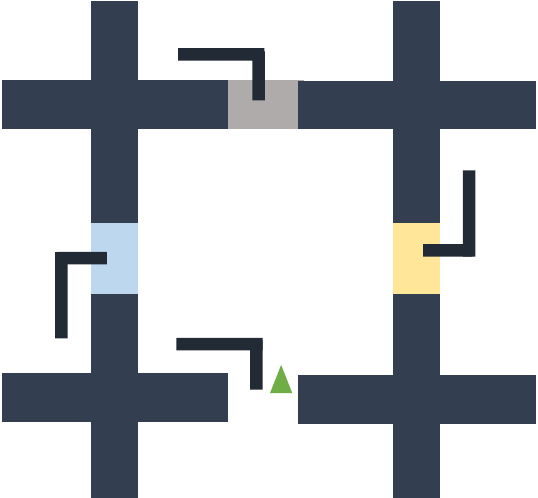}
    \caption{Rooms environment}
    \label{fig:rooms}
\end{figure}

The \textsc{Sac} subroutine runs the standard Soft Actor Critic algorithm for $N$ iterations on the subtask MDP $\M_{\sigma}^{\tilde{V}}$ (defined in Section~\ref{sec:reduction})\footnote{Note that it is possible to obtain samples from $\M_{\sigma}^{\tilde{V}}$ as long can one can obtain samples from $\M$ and membership in $F_{\sigma}$ can be decided.} with initial state distribution $\tilde{\eta}$ (defaults to $\eta$ if $D=\emptyset$). It returns the updated parameters along with states $X_{\sigma}$ visited during exploration with $X_{\sigma}\subseteq F_{\sigma}$. The states in $X_{\sigma}$ are used to update the initial state distribution for the next iteration following the Dataset Aggregation principle~\citep{ross2011reduction}.

\begin{algorithm}[t]
\begin{algorithmic}[1]
\Function{Arosac}{$\alpha$, $\beta$, $\delta$, $N$}
\State Initialize params $\Psi=\{\psi_{\sigma}\}_{\sigma\in\Sigma}$, $\Theta=\{\theta_{\sigma}\}_{\sigma\in\Sigma}$
\State Initialize target params $\Psi^{\targ}=\{\psi_{\sigma}^{\targ}\}_{\sigma\in\Sigma}$
\State Initialize replay buffers $\{\B_{\sigma}\}_{\sigma\in\Sigma}$ \State Initialize $D = \{\}$
\For{each iteration}
\State $\tilde{V}\gets\textsc{ObtainValueEstimator}(\Psi,\Theta)$
\For{$\sigma\in\Sigma$} \quad{\color{blue}// in parallel}
\State Run $\textsc{Sac}(\M_\sigma^{\tilde{V}},D,\psi_{\sigma}, \psi_{\sigma}^{\targ},\theta_{\sigma}, \alpha, \beta, \delta, N)$\label{line:vtilde}
\State Update $\psi_{\sigma}, \psi_{\sigma}^{\targ}, \theta_{\sigma}\ \text{and}\ X_\sigma$ to new values
\EndFor
\For{$\sigma\in\Sigma$}
\For{$s\in X_{\sigma}$}
\State $s'\sim T_{\sigma}(\cdot\mid s)$ and $D\gets D\cup\{s'\}$
\EndFor
\EndFor
\EndFor
\EndFunction
\end{algorithmic}
\caption{Asynchronous Robust Option \textsc{Sac}.\\
Inputs: Learning rates $\alpha$, entropy weight $\beta$, Polyak rate $\delta$ and number of \textsc{Sac} iterations $N$.}
\label{alg:aosac}
\end{algorithm}

\section{Experiments}\label{sec:experiments}
We evaluate\footnote{Our implementation is available online and can be found at \href{https://github.com/keyshor/rosac}{https://github.com/keyshor/rosac}.} our algorithms \textsc{Rosac} and \textsc{Arosac} on two multi-task environments; a rooms environment and an F1/10th racing car environment~\citep{f1tenth}.

\paragraph{Rooms environment.} 
We consider the environment shown in Figure~\ref{fig:rooms} which depicts a room with walls and exits. Initially the robot is placed in the green triangle. The L-shaped obstacles indicate walls within the room that the robot cannot cross. A state of the system is a position $(x,y)\in\R^2$ and an action is a pair $(v,\theta)$ where $v$ is the speed and $\theta$ is the heading angle to follow during the next time-step.

%The state of the system is the position $(x,y)$ of the robot and an action is given by $(v,\theta)\in\R^2$ where $v$ is the speed and $\theta$ is the heading angle.
There are three exits: left (blue), right (yellow) and up (grey) reaching each of which is a subtask. Upon reaching an exit, the robot enters another identical room where the exit is identified (via change-of-coordinates) with the bottom entry region of the current room. A task is a sequence of directions---e.g., $\code{left}\to\code{right}\to\code{up}\to\code{right}$ indicating that the robot should reach the left exit followed by the right exit in the subsequent room and so on. Although the dynamics are simple, the obstacles make learning challenging in the adversarial setting.

\begin{figure*}[t]
\centering
% \begin{subfigure}{0.29\textwidth}
% \includegraphics[width=\textwidth]{rooms_avg_jumps}
% \caption{Number of subtaks completed\\against random adversary}
% \end{subfigure}
\begin{subfigure}{0.28\textwidth}
\includegraphics[width=\textwidth]{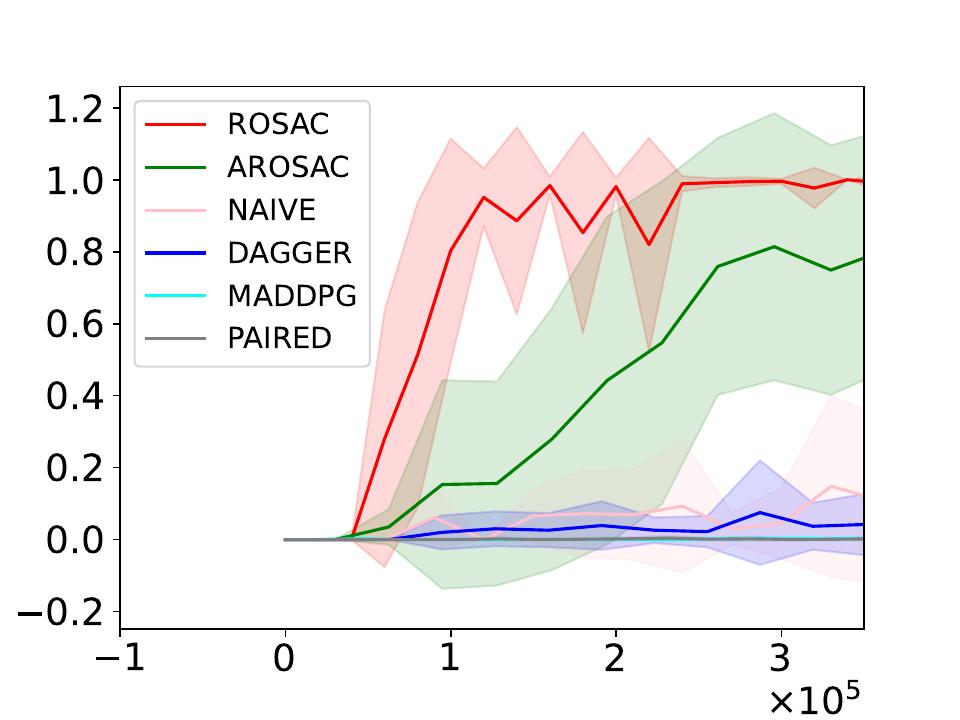}
\caption{Success probability against\\ random adversary}
\end{subfigure}\qquad\quad
\begin{subfigure}{0.28\textwidth}
\includegraphics[width=\textwidth]{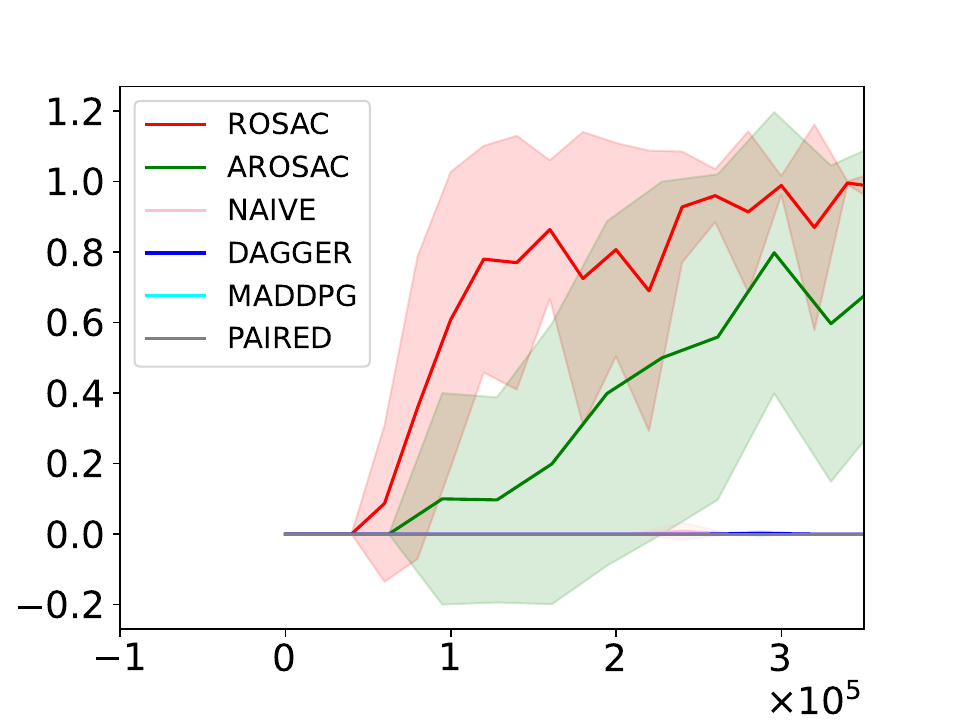}
\caption{Success probability against\\\textsc{Mcts} adversary}
\end{subfigure}\qquad\quad
\begin{subfigure}{0.28\textwidth}
\includegraphics[width=\textwidth]{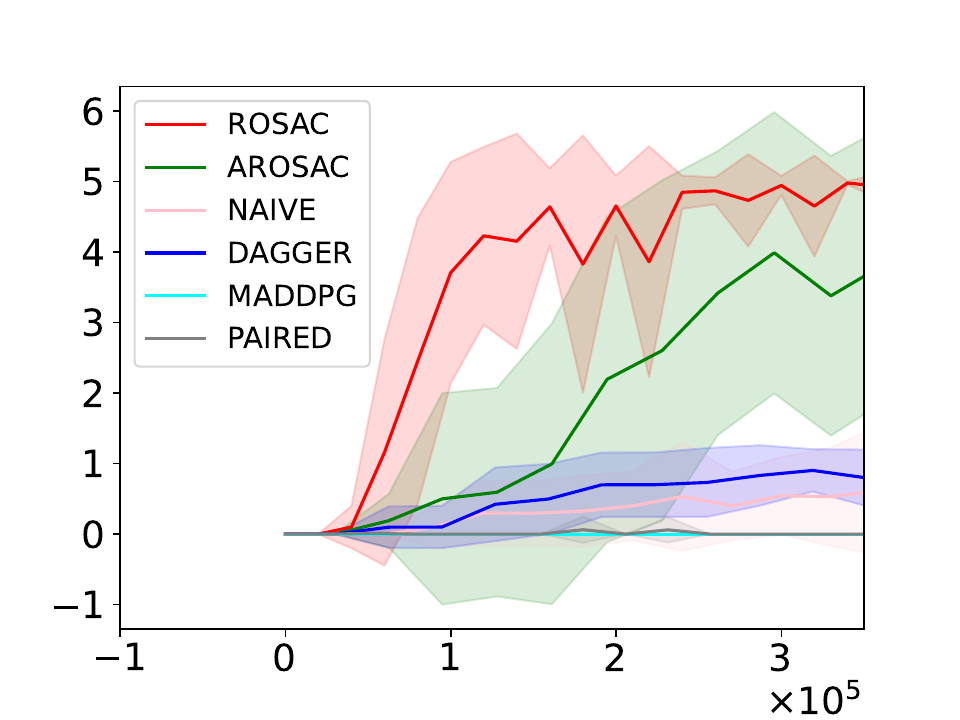}
\caption{Number of subtasks completed against \textsc{Mcts} adversary}
\end{subfigure}
\caption{Plots for the Rooms environment. $x$-axis denotes the number of sample steps and $y$-axis denotes either the average number of subtasks completed or the probability of completing 5 subtasks. Results are averaged over 10 runs. {Error bars indicate $\pm$ std.}}
\label{fig:rooms_plots}
\end{figure*}

\begin{figure*}[t]
\centering
% \begin{subfigure}{0.29\textwidth}
% \includegraphics[width=\textwidth]{f110_avg_jumps}
% \caption{Number of subtaks completed\\against random adversary}
% \end{subfigure}
\begin{subfigure}{0.28\textwidth}
\includegraphics[width=\textwidth]{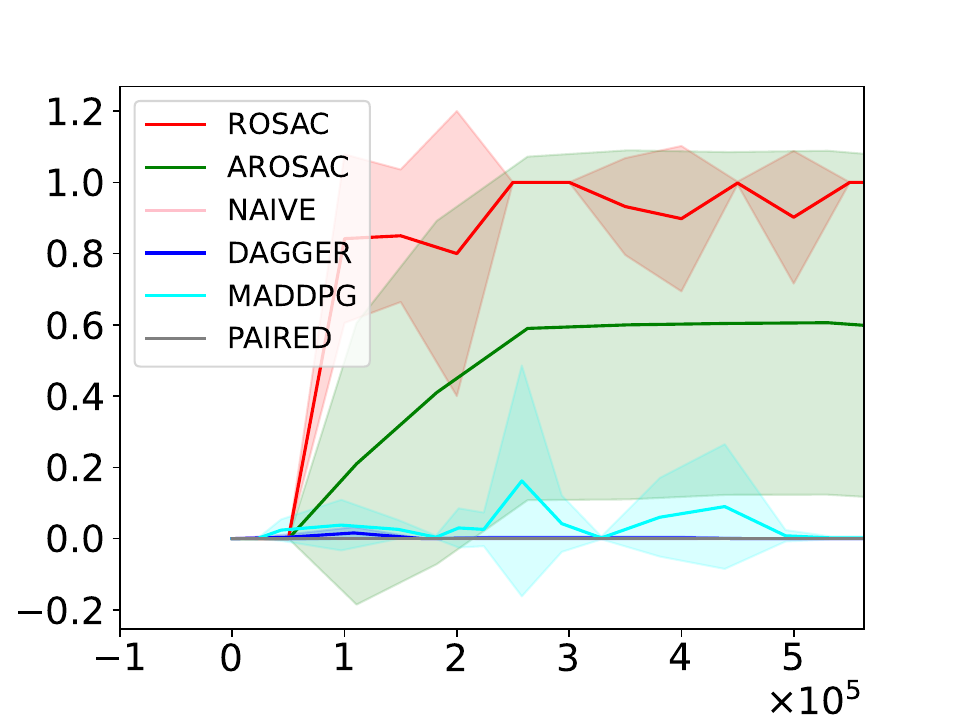}
\caption{Success probability against\\ random adversary}
\end{subfigure}\qquad\quad
\begin{subfigure}{0.28\textwidth}
\includegraphics[width=\textwidth]{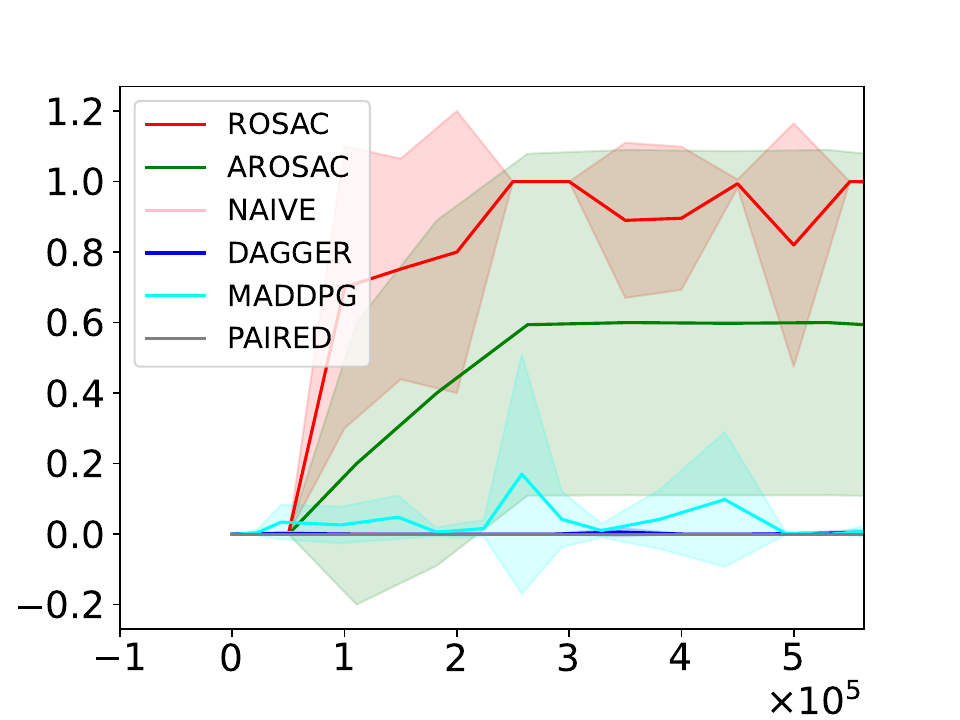}
\caption{Success probability against\\\textsc{Mcts} adversary}
\end{subfigure}\qquad\quad
\begin{subfigure}{0.28\textwidth}
\includegraphics[width=\textwidth]{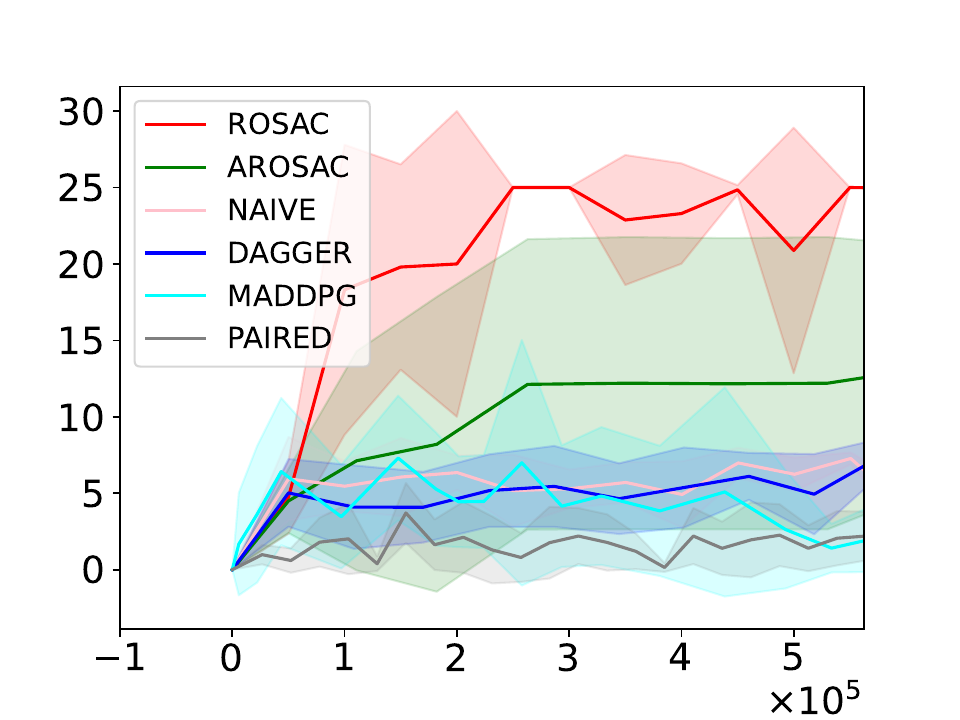}
\caption{Number of subtasks completed against \textsc{Mcts} adversary}
\end{subfigure}
\caption{Plots for the F1/10th environment. $x$-axis denotes the number of sample steps and $y$-axis denotes either the average number of subtasks completed or the probability of completing 25 subtasks. Results are averaged over 5 runs. {Error bars indicate $\pm$ std.}}
\label{fig:f110_plots}
\end{figure*}

\paragraph{F1/10th environment.} This is the environment in the motivating example. A publicly available simulator~\citep{f1tenth} of the F1/10th car is used for training and testing. The policies use the LiDAR measurements from the car as input (as opposed to the state) and we assume that the controller can detect the completion of each segment; as shown in prior work~\citep{ivanov2021compositional}, one can train a separate neural network to predict subtask completion.

\paragraph{Baselines.} We compare our approach to four baselines. The baseline \textsc{Naive} trains one controller for each subtask with the only aim of completing the subtask, similar to the approach used by~\citet{ivanov2021compositional}, using a manually designed initial state distribution.
\textsc{Dagger} is a similar approach which, instead of using a manually designed initial state distribution for training, infers the initial state distribution using the Dataset Aggregation principle~\citep{ross2011reduction}. In other words, \textsc{Dagger} is similar to \textsc{Arosac} except that $\M$ is used instead of $\M_{\sigma}^{\tilde{V}}$ (Line~\ref{line:vtilde} of Algorithm~\ref{alg:aosac}) for training the options in each iteration.
The \textsc{Maddpg} baseline solves the game $\G$ using the multi-agent RL algorithm proposed by~\citet{lowe2017multi} for solving concurrent Markov games with continuous states and actions. We use the open-source implementation of \textsc{Maddpg} by the authors.
{The \textsc{Paired} baseline refers to the unsupervised environment design approach proposed by~\citet{dennis2020emergent} in which the multi-task RL problem is viewed as a two-agent zero-sum game similar to our approach (but is not designed for long-horizon and compositional tasks). We implemented \textsc{Paired} with \textsc{Ppo} as the base algorithm to train the policies of the protagonist and the antagonist which consist of one NN policy per subtask. The adversary is parameterized by the probabilities (logits) of choosing various subtasks at different timesteps and is trained using \textsc{Reinforce} updates. }

\begin{figure*}[t]
\centering
% \begin{subfigure}{0.29\textwidth}
% \includegraphics[width=\textwidth]{rooms_avg_jumps}
% \caption{Number of subtaks completed\\against random adversary}
% \end{subfigure}
\begin{subfigure}{0.28\textwidth}
\includegraphics[width=\textwidth]{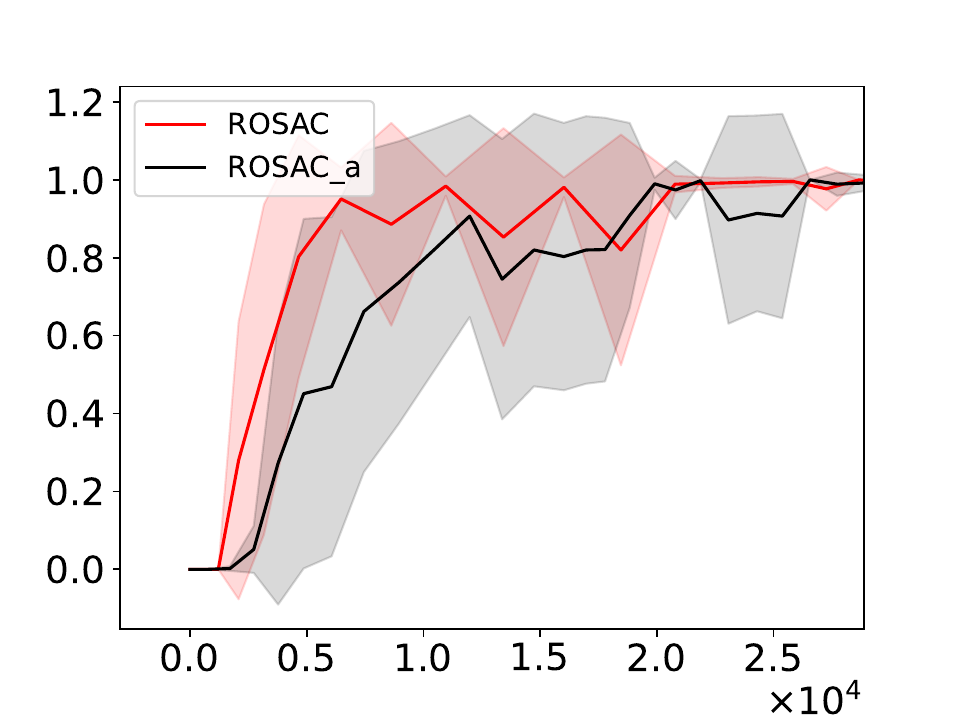}
\caption{Success probability against\\ random adversary}
\end{subfigure}\qquad\quad
\begin{subfigure}{0.28\textwidth}
\includegraphics[width=\textwidth]{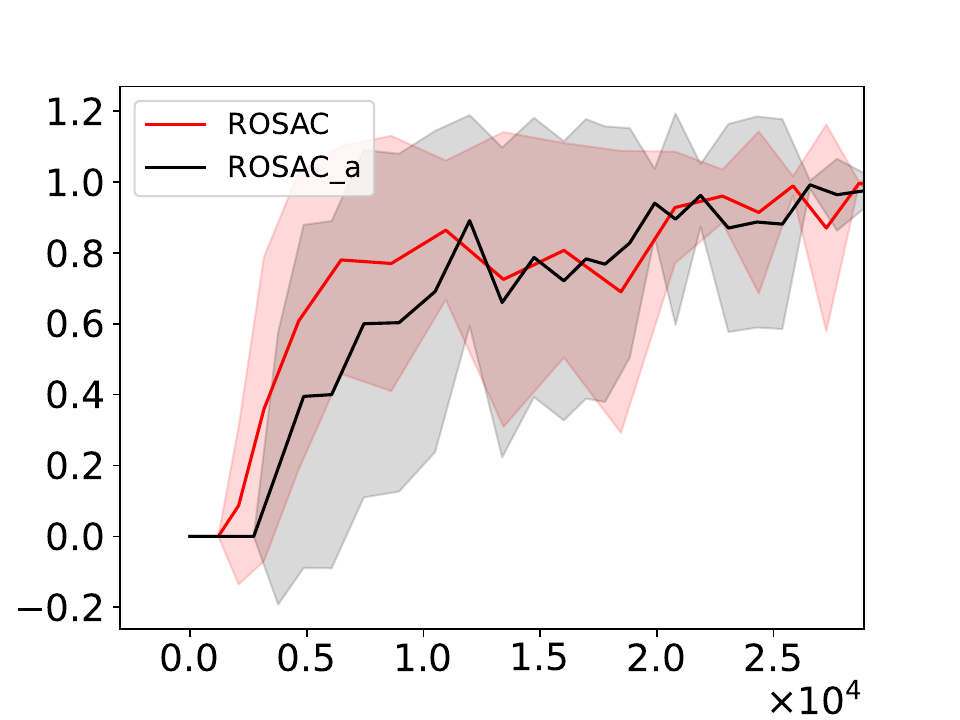}
\caption{Success probability against\\\textsc{Mcts} adversary}
\end{subfigure}\qquad\quad
\begin{subfigure}{0.28\textwidth}
\includegraphics[width=\textwidth]{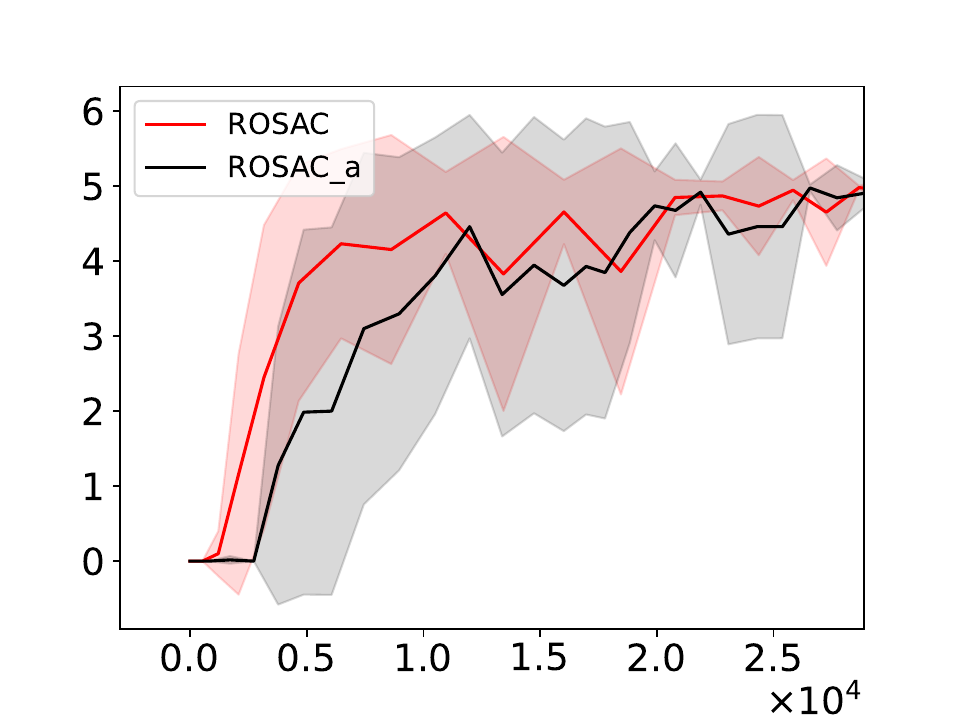}
\caption{Number of subtasks completed against \textsc{Mcts} adversary}
\end{subfigure}
\caption{Ablation study in the Rooms environment. $x$-axis denotes the number of sample steps and $y$-axis denotes either the average number of subtasks completed or the probability of completing 5 subtasks. Results are averaged over 10 runs. {Error bars indicate $\pm$ std.}}
\label{fig:rooms_ablation}
\end{figure*}

\paragraph{Evaluation.} We evaluate the performance of these algorithms against two adversaries. One adversary is the random adversary which picks the next subtask uniformly at random from the set of all subtasks. The other adversary estimates the worst sequence of subtasks for a given set of options using Monte Carlo Tree Search algorithm (\textsc{Mcts}) \citep{kocsis2006bandit}. The \textsc{Mcts} adversary is trained by assigning a reward of $1$ if it selects a subtask which the corresponding policy is unable to complete within a fixed time budget and a reward of $0$ otherwise. For the Rooms environment, we consider subtask sequences of length atmost 5 whereas for the F1/10th environment, we consider sequences of subtasks of length at most 25. We evaluate both the average number of subtasks completed as well as the probability of completing the set maximum number of subtasks.

\paragraph{Results.} The plots for the rooms environment are shown in Figure~\ref{fig:rooms_plots} and plots for the F1/10th environment are shown in Figure~\ref{fig:f110_plots}. We can observe that \textsc{Rosac} and \textsc{Arosac} outperform other approaches and learn robust options. \textsc{Arosac} performs worse as compared to \textsc{Rosac}; however, it has the added benefit of being parallelizable. \textsc{Dagger} and \textsc{Naive} baselines are unable to learn policies that can be used to perform long sequences of subtasks; this is mostly due to the fact that they learn options that cause the system to reach states from which future subtasks are difficult to perform---e.g., in the rooms environment, the agent sometimes reaches the left half of the exits from where it is difficult to reach the right exit in the subsequent room. Although \textsc{Maddpg} uses the same reduction to two-player games as \textsc{Rosac}, it ignores all the structure in the game and treats it as a generic Markov game. As a result, it learns a separate NN policy for each player which leads to the issue of unstable training, primarily due to the non-stationary nature of the environment observed by either agent. As shown in the plots, this leads to poor performance when applied to the problem of learning robust options. {Similarly, \textsc{Paired} is unable to learn good options and is likely due to two reasons.
Firstly, it does not use a compositional or a hierarchical approach for training the options, which causes it to scale poorly to long-horizon tasks. Secondly, it relies on ``on policy" algorithms (such as \textsc{Ppo}) for training which are less efficient than \textsc{Sac} for our environments.}

{\paragraph{Random vs MCTS adversary.} Given that the number of subtasks is small in our benchmarks, the random adversary is able to select difficult sequences of subtasks often (eg., a right after a left in the Rooms environment), which explains the similarity in the observed performance against the two different adversaries. Nonetheless, even when the performance is measured against a random adversary, our approach significantly outperforms all baselines. An ablation study in the Rooms environment shows that selecting subtasks randomly during training is not as effective as the $\varepsilon$-greedy approach. In Figure~\ref{fig:rooms_ablation}, \textsc{Rosac}$_a$ denotes the ablation in which $\varepsilon$ is set to $1$ in line~\ref{line:argmin} of Algorithm~\ref{alg:osac}. As the plots show, although this ablation is able to learn good subtask policies, it requires more samples than \textsc{Rosac}. Note that the ablation still uses the worst-case target value estimate over future subtasks $\tilde{V}$ for training the Q-networks.}

\section{Related Work}\label{sec:related}
\paragraph{Options and Hierarchical RL.} The options framework~\citep{sutton1999between} is commonly used to model subtask policies as temporally extended actions. In hierarchical RL~\citep{nachum2018data,nachum2019near,kulkarni2016hierarchical,dietterich2000state,bacon2017option,tiwari2019natural}, options are trained along with a high-level controller that chooses the sequence of options (plan) to execute in order to complete a specific high-level task. In contrast, our approach aims to train options that work for multiple high-level plans. There is also work on discovering options---e.g., using intrinsic motivation~\citep{machado2017laplacian}, entropy maximization~\citep{eysenbach2018diversity}, semi-supervised RL~\citep{finn2017generalizing}, skill chaining~\citep{konidaris2009skill, bagaria2019option, bagaria2021skill}, expectation maximization~\citep{daniel2016probabilistic} and subgoal identification~\citep{stolle2002learning}. {Our approach is complementary to option discovery methods since our algorithm can be used to learn robust policies corresponding to the options which can be used in multiple contexts.} There has also been a lot of research on planning using learned options~\citep{abel2020vpsa, pmlr-v130-jothimurugan21a,precup1998theoretical,theocharous2004approximate,konidaris2014constructing}. Some hierarchical RL algorithms have also been shown to enable few-shot generalization~\citep{pmlr-v130-jothimurugan21a} to unseen tasks. Closely related to our work is the work on compositional RL in the multi-task setting~\citep{ivanov2021compositional} in which the subtask policies are trained using standard RL algorithms in a naive way without guarantees regarding worst-case performance.

\paragraph{Multi-task RL.} There has been some work on RL for zero-shot generalization~\citep{pmlr-v139-vaezipoor21a,oh2017zero, sohn2018hierarchical,Kuo2020EncodingFA,andreas2017modular}; however, in these works, the learning objective is to maximize average performance with respect to a fixed distribution over tasks as opposed to the worst-case.
{Most closely related to our work is the line of work on minimax/robust RL \citep{pinto2017robust,campero2020learning, dennis2020emergent}, which aims to train policies to maximize the worst-case performance across multiple tasks/environments. However, there are a few key differences between existing work and our approach. Firstly, existing methods train a single NN policy to perform multiple tasks as opposed to training composable options for subtasks. As shown in the experiments, training a single policy does not scale well to complex long-horizon tasks. Secondly, these approaches do not provide strong convergence guarantees whereas we provide convergence guarantees in the finite-state setting. Finally, we utilize the structure in the problem to infer an adversary directly from the value functions instead of training a separate adversary as done in previous approaches.} 

\paragraph{Skill Chaining.} {Research on skill chaining \citep{bagaria2019option, bagaria2021robustly, bagaria2021skill, lee2021adversarial} has led to algorithms for discovering and training options so that they compose well. However, in these approaches, the aim is to compose the options to perform a \emph{specific} task which corresponds to executing the options in specific order(s). In contrast, our objective is independent of specific tasks and seeks to maximize performace across all tasks. Also, such approaches, at best, only provide hierarchical optimality guarantees whereas our algorithm converges to the optimal policy with respect to our game formulation.} There has also been work on skill composition using transition policies~\citep{lee2018composing}; this method assumes that the subtask policies are fixed and learns one transition policy per subtask which takes the system from an end state to a ``favourable" initial state for the subtask. However, poorly trained subtask policies can lead to situations in which it is not possible to achieve such transitions. In contrast, our approach trains subtask policies which compose well without requiring additional transition policies.

\paragraph{Multi-agent RL.} There has been a lot of research on multi-agent RL algorithms~\citep{lowe2017multi, hu2003nash, hu1998multiagent, littman2001friend, pmlr-v54-perolat17a, prasad2015two, akchurina2008multi} including algorithms for two-agent zero-sum games~\citep{pmlr-v119-bai20a, wei2017online, littman1994markov}. In this paper, we utilize the specific structure of our game to obtain a simple algorithm that neither requires solving matrix games nor trains a separate policy for the adversary. Furthermore, we show that we can obtain an asynchronous RL algorithm which enables training options in parallel.

\section{Conclusions}\label{sec:conclusions}
We have proposed a framework for training robust options which can be used to perform arbitrary sequences of subtasks. We first reduce the problem to a two-agent zero-sum stagewise Markov game which has a unique structure. We utilized this structure to design two algorithms, namely \textsc{Rosac} and \textsc{Arosac}, and demonstrated that they outperform existing approaches for training options with respect to multi-task performance. One potential limitation of our approach is that the set of subtasks is fixed and has to be provided by the user. An interesting direction for future work is to address this limitation by combining our approach with option discovery methods.

{\paragraph{Acknowledgements.} We thank the anonymous reviewers for their helpful comments. This research was partially supported by ONR award N00014-20-1-2115.}

%\paragraph{Societal impacts.} Our work seeks to improve reinforcement learning for complex long-horizon tasks. Any progress in this direction would enable robotics applications both with positive impact---e.g., flexible and general-purpose manufacturing robotics, robots for achieving agricultural tasks, and robots that can be used to perform household chores---and with negative or controversial impact---e.g., military applications. These issues are inherent in all work seeking to improve the abilities of robots.

% \begin{ack}
% \end{ack}

\bibliography{main}
\bibliographystyle{plainnat}

\clearpage
\onecolumn
\appendix
%\pagenumbering{arabic}
\section{Proofs}\label{sec:proofs}
In this section, we prove the theorems presented in the paper through a series of lemmas. The proofs here are adaptations of proofs of similar results for zero-sum concurrent games with a single discount factor~\citep{patek1997stochastic}. 

\subsection{Definitions}
We first introduce some notation and definitions. Recall that we defined $S_1=\{(s,\sigma)\mid\sigma\notin F_\sigma\}$, $S_2=\{(s,\sigma)\mid\sigma\in F_\sigma\}$ and $\bar{S}=S_1\cup S_2$. $\V=\{V:S_1\to\R\}$ denotes the set of value functions over $S_1$ and $\bar{\V}=\{V:\bar{S}\to\R\}$ denotes the set of value functions over $\bar{S}$. We use $\norm{\cdot}$ to denote the $\ell_{\infty}$-norm. For any $V\in\V$, $(s,\sigma)\in S_2$ and any $\sigma'\in\Sigma$, define $$\ext{V}_{\sigma'}(s,\sigma) = \sum_{s'\in S}T_{\sigma}(s'\mid s)V(s',\sigma').$$
Note that for any $(s,\sigma)\in S_2$, we have $\ext{V}(s,\sigma) = \min_{\sigma'\in\Sigma}\ext{V}_{\sigma'}(s,\sigma)$. Similarly, for any $V\in\V$, $(s,\sigma)\in S_1$ and $a\in A$, let us define $$\F_a(V)(s,\sigma) = \bar{R}((s,\sigma), a) + \gamma\cdot\sum_{s'\in S}P(s'\mid s,a)\ext{V}(s',\sigma)$$
with $\F(V)(s,\sigma) = \max_{a\in A}\F_a(V)(s,\sigma)$. Given any policy $\pi_1:\bar{S}\to A_1$ for agent 1 in $\G$, we define the resulting MDP $\G(\pi_1) = (\bar{S}, A_2, P_{\pi_1}, R_{\pi_1}, \gamma)$ with states $\bar{S}$ and actions $\A_2=\Sigma$ as follows. The transition probability function is given by
$$
P_{\pi_1}((s',\sigma')\mid (s,\sigma), a_2) =
\begin{cases}
\bar{P}((s',\sigma')\mid (s,\sigma), \pi_1(s,\sigma), a_2) &\text{if}\ (s,\sigma)\in S_1\\
\sum_{s''\in S}T_{\sigma}(s''\mid s)\bar{P}((s',\sigma')\mid (s'',a_2), \pi_1(s'',a_2), a_2) &\text{if}\ (s,\sigma)\in S_2
\end{cases}
$$
and the reward function is given by
$$
R_{\pi_1}((s,\sigma), a_2) =
\begin{cases}
-\bar{R}((s,\sigma), \pi_1(s,\sigma)) &\text{if}\ (s,\sigma)\in S_1\\
-\sum_{s'\in S}T_{\sigma}(s'\mid s)\bar{R}((s',a_2), \pi_1(s',a_2)) &\text{if}\ (s,\sigma)\in S_2.
\end{cases}
$$
Intuitively, the MDP $\G(\pi_1)$ merges every step of $\G$ in which a change of subtask occurs with the subsequent step in the environment, while using $\pi_1$ to choose actions for agent 1. For any $\bar{s}\in\bar{S}$, let $\D^{\G(\pi_1)}_{\bar{s}}(\pi_2)$ denote the distribution over infinite trajectories generated by $\pi_2$ starting at $\bar{s}$ in $\G(\pi_1)$. Then we define the value function for the MDP $\G(\pi_1)$ using
$$V_{\G(\pi_1)}^{\pi_2}(\bar{s})=\E_{\rho\sim\D^{\G(\pi_1)}_{\bar{s}}(\pi_2)}\Big[\sum_{t=0}^\infty\gamma^t R_{\pi_1}(\bar{s}_t, a_t)\Big]$$
for all $\pi_2:\bar{S}\to A_2$ and $\bar{s}\in\bar{S}$.

\subsection{Necessary Lemmas}
We need a few intermediate results in order to prove the main theorems. We begin by analyzing the operators $\ext{\cdot}:\V\to\bar{\V}$ and $\F:\V\to\V$ defined in Section~\ref{sec:vi}.

\begin{lemma}\label{lem:ext}
For any $V_1,V_2\in\V$, we have $\norm{\ext{V_1}-\ext{V_2}} = \norm{V_1-V_2}$.
\end{lemma}
\begin{proof}
For any $(s,\sigma)\in S_1$, we have $\nmod{\ext{V_1}(s,\sigma) - \ext{V_2}(s,\sigma)} = \nmod{V_1(s,\sigma) - V_2(s,\sigma)}$. For any $(s,\sigma)\in S_2$ and $\sigma'\in\Sigma$ we have
\begin{align*}
\nmod{\ext{V_1}_{\sigma'}(s, \sigma) - \ext{V_2}_{\sigma'}(s,\sigma)} &= \norm{\sum_{s'\in S}T_{\sigma}(s'\mid s)V_1(s',\sigma') - \sum_{s'\in S}T_{\sigma}(s'\mid s)V_2(s',\sigma')}\\
&\leq \sum_{s'\in S}T_{\sigma}(s'\mid s)\nmod{V_1(s',\sigma')-V_2(s',\sigma')}\\
&\leq \norm{V_1-V_2}.
\end{align*}
Now we have $\nmod{\ext{V_1}(s,\sigma) - \ext{V_2}(s,\sigma)} = \nmod{\min_{\sigma'}\ext{V_1}_{\sigma'}(s,\sigma) - \min_{\sigma'}\ext{V_2}_{\sigma'}(s,\sigma)}\leq \norm{V_1-V_2}$ which concludes the proof.
\end{proof}

Now we are ready to show that $\F$ is a contraction.

\begin{lemma}\label{lem:contr}
$\F:\V\to\V$ is a contraction mapping w.r.t the norm $\norm{\cdot}$.
\end{lemma}
\begin{proof}
Let $V_1,V_2\in \V$. Then for any $(s,\sigma)\in S_1$ and $a\in A$, 
\begin{align*}
    \nmod{\F_a(V_1)(s,\sigma)-\F_a(V_2)(s,\sigma)} &= \nmod{\gamma\sum_{s'\in S}P(s'\mid s,a)\ext{V_1}(s',\sigma) - \gamma\sum_{s'\in S}P(s'\mid s,a)\ext{V_2}(s',\sigma)}\\
    &\leq \gamma\sum_{s'\in S}P(s'\mid s,a)\nmod{\ext{V_1}(s',\sigma) - \ext{V_2}(s',\sigma)}\\
    &\leq \gamma\norm{V_1-V_2}
\end{align*}
where the last inequality followed from Lemma~\ref{lem:ext}. Therefore,  for any $(s,\sigma)\in S_1$, we have $\nmod{\F(V_1)(s,\sigma)-\F(V_2)(s,\sigma)} = \nmod{\max_{a}\F_a(V_1)(s,\sigma)-\max_{a}\F_a(V_2)(s,\sigma)}\leq \gamma\norm{V_1-V_2}$ showing that $\F$ is a contraction.
\end{proof}

Now we connect the value function of the game $\G$ with that of the MDP $\G(\pi_1)$.

\begin{lemma}\label{lem:gpi}
For any $\pi_1:\bar{S}\to A_1$, $\pi_2:\bar{S}\to A_2$ and $\bar{s}\in\bar{S}$, $V^{\pi_1,\pi_2}(\bar{s}) = -V_{\G(\pi_1)}^{\pi_2}(\bar{s})$.
\end{lemma}
\begin{proof}
Given an infinite trajectory $\bar{\rho} = \bar{s}_0a_0^1a_0^2\bar{s}_1a_1^1a_1^2\ldots$ in $\G$ we define a corresponding trajectory $\bar{\rho}_2=\bar{s}_{i_0}a_{i_0}^2\bar{s}_{i_1}a_{i_1}^2$ in $\G(\pi_1)$ as a subsequence where $i_0 = 0$ and $i_{t+1} = i_t+1$ if $\bar{s}_{i_t}\in S_1$ and $i_{t+1}=i_t+2$ if $\bar{s}_{i_t}\in S_2$. Then for any $\bar{s}\in \bar{S}$ we have
\begin{align*}
    V^{\pi_1, \pi_2}(\bar{s}) &= \E_{\bar{\rho}\sim\D^{\G}_{\bar{s}}({\pi}_1,{\pi}_2)}\Big[\sum_{t=0}^{\infty}\big(\prod_{k=0}^{t-1}\bar{\gamma}(\bar{s}_k)\big)\bar{R}(\bar{s}_t,a_t^1)\Big]\\
    &\stackrel{(1)}{=} \E_{\bar{\rho}\sim\D^{\G}_{\bar{s}}({\pi}_1,{\pi}_2)}\Big[\sum_{t=0}^{\infty}\gamma^t{R}_{\pi_1}(\bar{s}_{i_t},a_{i_t}^2)\Big]\\
    &\stackrel{(2)}{=} -\E_{{\rho}\sim\D^{\G(\pi_1)}_{\bar{s}}({\pi}_2)}\Big[\sum_{t=0}^{\infty}\gamma^t{R}_{\pi_1}(\bar{s}_{t},a_{t})\Big]\\
    &= -V_{\G(\pi_1)}^{\pi_2}(\bar{s})
\end{align*}
where (1) followed from the definitions of $\bar{\gamma}$ and $R_{\pi_1}$ and the fact that $\bar{R}(\bar{s}_t, a_t^1)=0$ if $\bar{s}_t\in S_2$, and (2) followed from the fact that sampling a trajectory ${\rho}$ by first sampling $\bar{\rho}$ from $\D^{\G}_{\bar{s}}({\pi}_1,{\pi}_2)$ and then constructing the subsequence $\bar{\rho}_2$ is the same as sampling an infinite trajectory ${\rho}$ from $\D^{\G(\pi_1)}_{\bar{s}}({\pi}_2)$\footnote{This can be shown formally by analyzing the probabilities assigned by the two distributions on cylinder sets.}.
\end{proof}

Lemma~\ref{lem:contr} shows that for any $V\in\V$ we have $$\lim_{n\to\infty}\F^n(V) = V_{\lim}$$ where $V_{\lim}\in\V$ is the unique fixed point of $\F$. Now we define two policies $\pi_1^*$ and $\pi_2^*$ for agents 1 and 2 respectively, as follows. For $(s,\sigma)\in S_1$ we have \begin{equation}\label{eq:pi1_def}\pi_1^*(s,\sigma) \in \operatorname{\arg\max}_{a\in A}\F_a(V_{\lim})(s,\sigma)\end{equation} and for $(s,\sigma)\in S_2$, we have \begin{equation}\label{eq:pi2_def}\pi^*_2(s,\sigma)\in\operatorname{\arg\min}_{\sigma'}\ext{V_{\lim}}_{\sigma'}(s,\sigma).\end{equation}
Note that the actions taken by $\pi_1^*$ in $S_2$ and $\pi_2^*$ in $S_1$ can be arbitrary since they do not affect the transitions of the game $\G$. Now we show that for any $\bar{s}\in\bar{S}$, $\pi^*_1$ maximizes $V^{\pi_1, \pi_2^*}(\bar{s})$ and $\pi^*_2$ minimizes $V^{{\pi}^*_1, {\pi}_2}(\bar{s})$.

\begin{lemma}\label{lem:pi2}
For any $\bar{s}\in \bar{S}$, $V^{\pi^*_1,\pi^*_2}(\bar{s}) = \min_{{\pi}_2}V^{\pi^*_1,{\pi}_2}(\bar{s}) = \ext{V_{\lim}}(\bar{s})$. 
\end{lemma}
\begin{proof}
Let $\G(\pi_1^*) = (\bar{S},\A_2,P_{\pi_1^*},R_{\pi_1^*},\gamma)$. For any $(s,\sigma)\in S_2$, we have
\begin{align*}
    \ext{V_{\lim}}(s,\sigma) &= \min_{\sigma'\in\Sigma}\sum_{s'\in S}T_{\sigma}(s'\mid s)\ext{V_{\lim}}(s',\sigma')\\
    &\stackrel{(3)}{=} \min_{\sigma'\in\Sigma}\sum_{s'\in S}T_{\sigma}(s'\mid s)\Big(\bar{R}((s',\sigma'),a) + \gamma\cdot\sum_{s''\in S}P(s'\mid s,a)\ext{V_{\lim}}(s'',\sigma')\Big)\Big|_{a=\pi^*_1(s',\sigma')}\\
    &\stackrel{(4)}{=} \min_{a_2\in A_2}\Big(-R_{\pi_1^*}((s,\sigma),a_2) + \gamma\sum_{\bar{s}\in\bar{S}}P_{\pi_1^*}(\bar{s}\mid (s,\sigma),a_2)\ext{V_{\lim}}(\bar{s})\Big)
\end{align*}
where (3) followed from the definitions of $V_{\lim}$ and $\pi_1^*$ and (4) followed from the definitions of $R_{\pi_1^*}$ and $P_{\pi_1^*}$. Since $-\ext{V_{\lim}}$ satisfies the Bellman equations for the MDP $\G(\pi_1^*)$, the optimal value function for $\G(\pi_1^*)$ is given by $V^*_{\G(\pi_1^*)} = -\ext{V_{\lim}}$. Now, from the definition of $\pi_2^*$ we can conclude that $\pi_2^*$ is an optimal policy for $\G(\pi_1^*)$. Therefore, Lemma~\ref{lem:gpi} implies that for any $\bar{s}\in\bar{S}$, $$\min_{\pi_2}V^{\pi_1^*, \pi_2}(\bar{s}) = \min_{\pi_2}-V_{\G(\pi_1^*)}^{\pi_2}(\bar{s}) = -V_{\G(\pi_1^*)}^{\pi_2^*}(\bar{s}) = V^{\pi_1^*, \pi_2^*}(\bar{s}).$$
Hence, we have proved the desired result.
\end{proof}

The following lemma can be shown using a similar argument and the proof is omitted.
\begin{lemma}\label{lem:pi1}
For any $\bar{s}\in \bar{S}$, $V^{\pi^*_1,\pi^*_2}(\bar{s}) = \max_{{\pi}_1}V^{\pi_1,{\pi}_2^*}(\bar{s}) = \ext{V_{\lim}}(\bar{s})$. 
\end{lemma}

The following lemma shows that it does not matter which agent picks its policy first.
\begin{lemma}\label{lem:best_policy}
For any policies ${\pi}^*_1$ and ${\pi}^*_2$ satisfying Equations~\ref{eq:pi1_def} and~\ref{eq:pi2_def} respectively, for all $\bar{s}\in\bar{S}$,
$$V^{{\pi}^*_1, {\pi}^*_2}(\bar{s}) = \min_{\pi_2}\max_{\pi_1}V^{\pi_1,\pi_2}(\bar{s})=\max_{\pi_1}\min_{\pi_2}V^{\pi_1,\pi_2}(\bar{s}) = V^*(\bar{s}).$$
\end{lemma}
\begin{proof}

We have, for any $\bar{s}\in\bar{S}$,
\begin{align*}
    V^*(\bar{s})&=\max_{\pi_1}\min_{\pi_2}V^{\pi_1,\pi_2}(\bar{s})\\
    &\geq \min_{\pi_2}V^{\pi_1^*,\pi_2}(\bar{s})\\
    &\stackrel{(1)}{=} V^{\pi_1^*,\pi_2^*}(\bar{s})\\
    &\stackrel{(2)}{=} \max_{\pi_1}V^{\pi_1, \pi_2^*}(\bar{s})\\
    &\geq \min_{\pi_2}\max_{\pi_1}V^{\pi_1,\pi_2}(\bar{s})\\
    &\geq \max_{\pi_1}\min_{\pi_2}V^{\pi_1,\pi_2}(\bar{s})\\
    &= V^*(\bar{s})
\end{align*}
where the (1) followed from Lemma~\ref{lem:pi2} and (2) followed from Lemma~\ref{lem:pi1}.
\end{proof}

\subsection{Proof of Theorem~\ref{thm:reduction}}
Let $\Pi(\pi_1) = \{\pi_{\sigma}\mid \sigma\in\Sigma\}$ be the set of subtask policies defined by $\pi_1$. Let $\tau=\sigma_0\sigma_1\ldots$ be a task. Then we define a history-dependent policy $\pi_2^\tau$ in $\G(\pi_1)$ which maintains an index $i$ denoting the current subtask and picks $\sigma_{i+1}$ upon reaching any state in $S_2$ while simultaneously updating the index to $i+1$. Then we have
\begin{align*}
    J(\Pi(\pi_1)) &= \inf_{\tau\in\T}\E_{\rho\sim\D_{\tau}^{\Pi}}\Big[\sum_{t=0}^\infty\gamma^tR_{\tau[i_t]}(s_t, \pi_{\tau[i_t]}(s_t))\Big]\\
    &\stackrel{(1)}{=} \inf_{\tau\in\T}\E_{\bar{s}\sim\bar{\eta}}\Big[\E_{\rho\sim\D_{\bar{s}}^{\G(\pi_1)}(\pi_2^{\tau})}\big[-\sum_{t=0}^{\infty}\gamma^t R_{\pi_1}(\bar{s}_t, a_t)\big]\Big]\\
    &= \inf_{\tau\in\T}\E_{\bar{s}\sim\bar{\eta}}[-V_{\G(\pi_1)}^{\pi_2^\tau}(\bar{s})]\\
    &\geq\E_{\bar{s}\sim\bar{\eta}}[-\sup_{\tau\in\T}V_{\G(\pi_1)}^{\pi_2^\tau}(\bar{s})]\\
    &\stackrel{(2)}{\geq}\E_{\bar{s}\sim\bar{\eta}}[-\max_{\pi_2}V_{\G(\pi_1)}^{\pi_2}(\bar{s})]\\
    &\stackrel{(3)}{=}\E_{\bar{s}\sim\bar{\eta}}[\min_{\pi_2}V^{\pi_1,\pi_2}(\bar{s})]\\
    &= J_\G(\pi_1)
\end{align*}
where (1) followed from the definitions of $\pi_2^\tau$ and $\G(\pi_1)$, (2) followed from the fact that there is an optimal stationary policy maximizing $V_{\G(\pi_1)}^{\pi_2}(\bar{s})$ and (3) followed from Lemma~\ref{lem:pi2}.

\subsection{Proof of Theorem~\ref{thm:opt_policy}}
Since $V^*=\ext{V_{\lim}}$, for all $(s,\sigma)\in\bar{S}$ we have $\pi_1^*(s,\sigma) \in \operatorname{\arg\max}_{a\in A}\F_a(V_{\lim})(s,\sigma)$. Now for any $\pi_2^*$ satisfying Equation~\ref{eq:pi2_def}, we can conclude from Lemma~\ref{lem:best_policy} that, for any $\bar{s}\in\bar{S}$,
\begin{align*}
    J_{\G}(\pi_1^*)&=\E_{\bar{s}\sim\bar{\eta}}[\min_{\pi_2}V^{\pi_1^*,\pi_2}(\bar{s})]\\
   &= \E_{\bar{s}\sim\bar{\eta}}[V^{\pi_1^*,\pi_2^*}(\bar{s})]\\
   &= \E_{\bar{s}\sim\bar{\eta}}[\max_{\pi_1}\min_{\pi_2}V^{\pi_1,\pi_2}(\bar{s})]\\
   &\geq \max_{\pi_1}\E_{\bar{s}\sim\bar{\eta}}[\min_{\pi_2}V^{\pi_1,\pi_2}(\bar{s})]\\
   &= \max_{\pi_1}J_{\G}(\pi_1)
\end{align*}
which shows that $\pi^*_1$ maximizes $J_{\G}(\pi_1)$.\hfill\qed

\subsection{Proof of Theorem~\ref{thm:fp}}
From Lemma~\ref{lem:contr}, we can conclude that $\F$ is a contraction over $\V$ w.r.t. the $\ell_{\infty}$-norm. Lemmas~\ref{lem:best_policy} and~\ref{lem:pi2} gives us that $V^*\downarrow_{S_1} = V_{\lim}$. Now the definition of $V_{\lim}$ implies that $\lim_{n\to\infty}\F^n(V) = V^*\downarrow_{S_1}$ for all $V\in\V$.\hfill\qed

\subsection{Proof of Theorems~\ref{thm:async} and~\ref{thm:async_partial}}
This proof is similar to the proof of convergence of asynchronous value iteration for MDPs presented in the book by~\citet{BertsekasTsitsiklis96}. It is easy to see that, for any $V\in\V$ and $\sigma\in\Sigma$, the operators $\ext{\cdot}$, $\F$, $\F_{\async}$, and $\F_{\sigma, V}$ are monotonic. Recall that, for any $V\in\V$ and $\sigma\in\Sigma$, we defined the corresponding ${V}_\sigma\in\V_{\sigma}$ using ${V}_{\sigma}(s) = \ext{V}(s,\sigma)$ if $s\in S$ and ${V}_{\sigma}(\bot) = 0$. Also, we have $\F(V)(s,\sigma) = \F_{\sigma, V}(V_{\sigma})(s) = \F_1(V)(s,\sigma)$ for all $(s,\sigma)\in S_1$.

Now let $V\in\V$ be a value function such that $\F(V) \leq V$. Then we have $\F_{\sigma,V}(V_\sigma)\leq V_{\sigma}$ for all $\sigma\in\Sigma$. Therefore, using monotonicity of $\F_{\sigma,V}$, we get that $\F_{\sigma,V}^m(V_{\sigma})\leq \F_{\sigma,V}^{m-1}(V_{\sigma})\leq V_{\sigma}$ for all $m>0$ which implies $\F_m(V)\leq\F_{m-1}(V)\leq V$. Hence, for any $(s,\sigma)\in S_1$,
\begin{align*}
    \F_{\async}(V)(s,\sigma) &= \W_{\sigma}(V)(s)
    =\lim_{m\to\infty}\F_{\sigma,V}^m(V_\sigma)(s)\\
    &\leq \F_{\sigma,V}(V_{\sigma})(s) = \F(V)(s,\sigma).
\end{align*}
Furthermore, letting $V^m = \F_m(V)$ we get that $\ext{V^m}\leq \ext{V}$ and hence $V^m_{\sigma}\leq \F_{\sigma,V}^{m}(V_{\sigma})$ for all $\sigma\in\Sigma$. Also, for $(s,\sigma)\in S_1$, $\F(V^m)(s,\sigma)=\F_{\sigma, V^m}(V^m_{\sigma})(s)\leq \F_{\sigma,V}(V^m_{\sigma})(s)\leq\F_{\sigma,V}^{m+1}(V_{\sigma})(s) = V^{m+1}(s,\sigma)$. Therefore, using continuity of $\F$, we have $\F(\F_{\async}(V)) = \F(\lim_{m\to\infty}V^m) = \lim_{m\to\infty}\F(V^m)\leq \lim_{m\to\infty}V^{m+1} = \F_{\async}(V)$. Now we can show by induction on $n$ that, for any $V\in\V$ with $\F(V)\leq V$ and $n\geq 1$, we have $\F(\F_{\async}^n(V))\leq \F_{\async}^n(V)$ and
$$V_{\lim}\leq \F_{\async}^n(V)\leq \F^n(V).$$
Taking the limit as $n\to\infty$ gives us that $\lim_{n\to\infty}\F_{\async}^n(V) = V_{\lim}$ if $\F(V)\leq V$. Using a symmetric argument, we get that $\lim_{n\to\infty}\F_{\async}^n(V) = V_{\lim}$ if $\F(V)\geq V$.

Let $I\in\V$ be defined by $I(s,\sigma) = 1$ for all $(s,\sigma)\in S_1$. For a general $V\in\V$, we can find a $\delta>0$ such that we have $V^-=V_{\lim}-\delta I\leq V\leq V_{\lim}+\delta I = V^+$ and $\F(V^-)\geq V^-$ and $\F(V^+)\leq V^+$. Therefore, using monotonicity of $\F_{\async}$ we get
$$\F_{\async}^n(V^-)\leq \F_{\async}^n(V)\leq\F_{\async}^n(V^+)$$
for all $n\geq 0$. Taking the limit as $n$ tends to $\infty$ gives us the required result. Theorem~\ref{thm:async_partial} follows from a similar argument.\hfill\qed

\subsection{Proof of Theorem~\ref{thm:q_learning}}
Given a function $Q:S_1\times A\to \R$ we define a new function $\H(Q)$ using
$$\H(Q)(s,\sigma, a) = \bar{R}((s,\sigma), a) + \gamma\sum_{s'\in S}P(s'\mid s,a)\ext{V_Q}(s',\sigma)$$
for all $(s,\sigma)\in S_1$ and $a\in A$. Then, Robust Option $Q$-learning is of the form 
$$Q_{t+1}(s,\sigma,a) = (1-\alpha_t(s,\sigma,a))Q_t(s,\sigma,a) + \alpha_t(s,\sigma,a)\Big(H(Q_t)(s,\sigma,a) + w_t(s,\sigma,a)\Big)$$
where the noise factor is defined by
$$w_t(s,\sigma,a) = \gamma\ext{V_{Q_t}}(\tilde{s},\sigma) - \gamma\sum_{s'\in S}P(s'\mid s,a)\ext{V_Q}(s',\sigma)$$
with $\tilde{s}\sim P(\cdot\mid s,a)$ being the observed sample. Let $\X_t$ denote the measure space generated by the set of random vectors $\{Q_0,Q_1,\ldots,Q_t,w_0,\ldots,w_{t-1},\alpha_0,\ldots,\alpha_t\}$. Then, for all $(s,\sigma)\in S_1$, $a\in A$ and $t\geq 0$, we have
$$\E[w_t(s,\sigma,a)\mid \X_t]=0$$
and
$$\E[w^2_t(s,\sigma,a)\mid \X_t]\leq 4\gamma^2\max_{s'\in S}\Big\{\ext{V_{Q_t}}^2(s',\sigma)\Big\}\leq 4\gamma^2\max_{(s',\sigma')\in S_1, a'\in A}\Big\{Q_t^2(s',\sigma',a')\Big\}.$$
Furthermore, using Lemmas~\ref{lem:ext} and~\ref{lem:contr} and the definition of $V_Q$ we can conclude that $\H$ is a contraction w.r.t the $\ell_\infty$-norm and $Q^*$ is the unique fixed point of $\H$. Therefore, the random sequence of $Q$-functions $\{Q_t\}_{t\geq0}$ satisfies all assumptions in Proposition 4.4 of~\citet{BertsekasTsitsiklis96} implying that $Q_t\to Q^*$ as $t\to\infty$ with probability 1.\hfill\qed

\section{Experimental Details}
All experiments were run on a 48-core machine with 512GB of memory and 8 GPUs. In all approaches (ours and baselines) except for \textsc{Maddpg}, the policy consists of one fully-connected NN per subtask, each with two hidden layers. \textsc{Maddpg} consists of two policies, one for the agent and one for the adversary, each with two hidden layers. In the case of \textsc{Maddpg}, the subtask is encoded in the observation using a one-hot vector. All hyperparameters were computed by grid search over a small set of values.

\paragraph{Rooms environment.} The hidden dimension used is 64 for all approaches except \textsc{Maddpg} for which we use 128 dimensional hidden layers. For \textsc{Dagger}, \textsc{Naive} and \textsc{Arosac} we run \textsc{Sac} with Adam optimizer (learning rate of $\alpha=0.01$), entropy weight $\beta=0.05$, Polyac rate $0.005$ and batch size of 100. In each iteration of \textsc{Arosac} and \textsc{Dagger}, \textsc{Sac} is run for $N=10000$ steps. Similarly, \textsc{Rosac} is run with Adam optimizer (learning rates $\alpha_\psi=\alpha_\theta=0.01$), entropy weight $\beta=0.05$, Polyac rate $0.005$ and batch size of 300. The \textsc{Maddpg} baseline uses a learning rate of $0.0003$ and batch size of 256. \textsc{Paired} uses \textsc{Ppo} with a learning rate of $0.02$, batch size of $512$, minibatch size of $128$ and $4$ epochs for each policy update. The adversary is trained using \textsc{Reinforce} with a learning rate of $0.003$.

\paragraph{F1/10th environment.} The hidden dimension used is 128 for all approaches. For \textsc{Dagger}, \textsc{Naive} and \textsc{Arosac} we run \textsc{Sac} with Adam optimizer (learning rate of $\alpha=0.001$), entropy weight $\beta=0.03$, Polyac rate $0.005$ and batch size of 128. In each iteration of \textsc{Arosac} and \textsc{Dagger}, \textsc{Sac} is run for $N=10000$ steps. Similarly, \textsc{Rosac} is run with Adam optimizer (learning rates $\alpha_\psi=\alpha_\theta=0.001$), entropy weight $\beta=0.03$, Polyac rate $0.005$ and batch size of $5\times 128$. The \textsc{Maddpg} baseline uses a learning rate of $0.0003$ and batch size of 256. \textsc{Paired} uses \textsc{Ppo} with a learning rate of $0.001$, batch size of $512$, minibatch size of $128$ and $4$ epochs for each policy update. The adversary is trained using \textsc{Reinforce} with a learning rate of $0.003$.

\end{document}